\documentclass[]{imsart}
\usepackage{amsmath,amssymb,natbib,graphicx,url,algorithm2e}
\usepackage[colorlinks=true]{hyperref}

\RequirePackage{mystyle}

\arxiv{1602.04265}

\begin{document}

\begin{frontmatter}

\title{Lasso Guarantees for Time Series Estimation under Subgaussian Tails and $ \beta $-Mixing
}
\runtitle{Regularized Estimation in High-Dimensional Time Series}

\begin{aug}
  \author{\fnms{Kam Chung}  \snm{Wong,}\corref{}
\ead[label=e1]{kamwong@umich.edu}}
    \author{\fnms{Zifan} \snm{Li} \thanksref{t2} \ead[label=e2]{zifan.li@yale.edu} }
  \and
  \author{\fnms{Ambuj} \snm{Tewari}\ead[label=e3]{tewaria@umich.edu}}
  \ead[label=u1,url]{http://www.foo.com}
\thankstext{t2}{Most of Zifan Li's contribution to this work occurred while he was an undergraduate student at the University of Michigan.}

  \runauthor{Wong et al.}

  \affiliation{University of Michigan}

  \address{
          \printead{e1,e2,e3}}

\end{aug}
\begin{abstract}
Many theoretical results on estimation of high dimensional time series require specifying an underlying data generating model (DGM). Instead, along the footsteps of~\cite{wong2017lasso}, this paper relies only on (strict) stationarity and $ \beta $-mixing condition to establish consistency of lasso when data comes from a $\beta$-mixing process with marginals having subgaussian tails. Because of the general assumptions, the data can come from DGMs different than standard time series models such as VAR or ARCH. When the true DGM is not VAR, the lasso estimates correspond to those of the best linear predictors using the past observations. We establish non-asymptotic inequalities for estimation and prediction errors of the lasso estimates. Together with~\cite{wong2017lasso}, we provide lasso guarantees that cover full spectrum of the parameters in specifications of $ \beta $-mixing subgaussian time series.  Applications of these results potentially extend to non-Gaussian, non-Markovian and non-linear times series models as the examples we provide demonstrate. In order to prove our results, we derive a novel Hanson-Wright type concentration inequality for $\beta$-mixing subgaussian random vectors that may be of independent interest.
\end{abstract}


\tableofcontents
\end{frontmatter}

\section{Introduction}\label{sec:intro}

Efficient estimation methods in high-dimensional statistics~\citep{buhlmann2011statistics,hastie2015statistical} include methods based on convex relaxation (see, e.g.,~\cite{chandrasekaran2012convex,negahban2012unified}) and methods using iterative optimization
techniques (see, e.g.,~\cite{beck2009fast,agarwal2012fast,donoho2009message}). A lot of work in the past decade has improved our understanding of the theoretical properties of these algorithms. However, the bulk of existing theoretical work focuses on \emph{iid samples}. The extension of theory and algorithms in high-dimensional statistics to \emph{time series data} is just beginning to occur as we briefly summarize in Section~\ref{sec:worksummary} below. Note that, in time series applications, \emph{dependence among samples} is the norm rather than the exception. So the development of high-dimensional statistical theory to handle dependence is a pressing concern in time series estimation.

The information age and scientific advances have led to explosions in large data sets, among which many exhibit temporal dependence. These can include, for example, data from micro-array experiments, dynamic social networks, mobile phone usage, high frequency stock market trading,  daily grocery sales, etc.  To gain insights into how  variables interact with each other over time and/or to do forecasting, it is important to do a systematic analysis on all the variables simultaneously.

The vector autoregressive (VAR) family is  popular choice to 
study the network of dynamic interactions among variables in high dimensions. Formally, given a $ p $-dimensional time series, $(X_t)$, where $ X_t= (X_t^1, \cdots, X_t^p) \in \R^p,\, \forall t $, and for iid innovations $ (\epsilon_t),\, \epsilon_t \in \R^p ,\, \forall t$, a VAR($ d $) model admits the representation
\begin{align*}
X_t = A_1 X_{t-1} + \cdots + A_d X_{t-d} + \epsilon_t, 
\end{align*}

Sims proposed using the VAR model as a theory-free model for Granger causality \cite{sims1980macroeconomics}. Theoretical foundations on using VAR comes from the Wold decomposition theorem which guarantees that any covariance-stationary time series can be approximated by a finite order autoregressive model (and a deterministic part). Empirically, VAR has proven to be a successful Granger causality framework in  domain science applications. The variables in the VAR can represent: economic variables from temporal panel data where the panel of subjects can be individuals, firms, households, etc. \cite{cao2011asymptotic, binder2005estimation}; macroeconomic variables, including government spending and taxes on economic output \cite{blanchard2002empirical}; stock price and volume \cite{hiemstra1994testing}; gene expressions in a dynamic regulatory network \cite{michailidis2013autoregressive}; or, regions of the brain from time course fMRI data \cite{krumin2010multivariate}.  

The set of coefficient matrices $ A_1, \cdots, A_d $ provides insights into the interrelationships among variables over time. For example, a non-zero $ [i,j] $-entry in $ A_k $  reflects that $ X^j $ likely has influence on $ X^i $ after $ k $-steps.
Under the high-dimensional settings, we are interested in a sparse predictor of the present observation using a linear combination of the past because we believe that not all variables will have a significant influence on every other variable. We consider the $ \ell_1 $-regularized least squares, or lasso, estimation of the problem. 
When the data are truly sampled from a VAR, the lasso estimates are those of the VAR transition matrices.
 Otherwise, the lasso estimates the best sparse linear predictor of $X_t$ in terms of $X_{t-d},\ldots, X_{t-1}$. Under stationarity (and finite 2nd moment conditions), the estimand is well defined even if the DGM is not a finite order VAR.

This paper provides finite sample parameter estimation and prediction error bounds for lasso in stationary processes with subgaussian marginals and  geometrically decaying $\beta$-mixing coefficients (Corollary~\ref{cor:subgauss}). A previous work \cite{wong2017lasso} proved lasso guarantees for $ \beta $-mixing times series with subweibull observations. To be specific, the subweibull parameter $ \gamma_2 $ measures rate  of probability tail decay while $ \gamma_1 $ quantifies dependence among observations in the series (Assumptions 6 and 7 in~\cite{wong2017lasso}). The pair $ (\gamma_1, \gamma_2) \in \R^2_+$ characterizes the difficulty landscape of the lasso problem. For example, $ \gamma_1 \rightarrow \infty $ (independence) and $ \gamma_2\rightarrow \infty $ (a.s. bounded) corresponds to an easy case while $ \gamma_1 \rightarrow 0 $ and $ \gamma_2\rightarrow 0$ a hard one. 
\cite{wong2017lasso} provided lasso guarantees for the sets of $ (\gamma_1, \gamma_2) $ such that $ \bpar{ \nicefrac{1}{\gamma_1} + \nicefrac{2}{\gamma_2}}^{-1} <1 $. In this paper, the lasso results pertain to geometrically $ \beta $-mixing ($ \gamma_1>0 $) time series with  subgaussian observations (equivalent to subweibull with $ \gamma_2=2 $). Together, we have lasso consistency results that cover the full spectrum of possibilities for the pair $ (\gamma_1, \gamma_2) \in \R^2_+$.

\subsection{Overview of the Paper}
This paper provides non-asymptotic lasso consistency guarantees of VAR estimation and prediction for data sampled from large classes of data-generating mechanisms (DGMs). This generalizes the current lasso theory from (1) Gaussian to subgaussian data, and from (2) requiring known parametric DGMs to weaker and more general mixing conditions which, roughly speaking, means that two observations far apart in time are approximately independent. The non-asymptotic rates of decay are close to being optimal. 
Our results rely on novel concentration inequality (Lemma~\ref{result: subgau:tailbdd}) for $\beta$-mixing subgaussian random variables that may be of independent interest. The inequality is proved by applying a blocking trick to Bernstein's concentration inequality for iid random variables. All proofs are deferred to the appendix.


These guarantees serve to show that we can safely employ the VAR framework to do estimation and/or prediction on high-dimensional data sampled from a wide range of DGMs.
To illustrate potential applications of our results, we present four examples. 
Example~\ref{ex:GaussVAR} considers a vanilla Gaussian VAR. Example~\ref{ex:sGVAR} considers VAR models with subgaussian innovations. 
Examples~\ref{ex:sGmisVAR} is concerned with subgaussian VAR models when the model is mis-specified. Finally, we go beyond linear models and introduce non-linearity in the DGM in Example~\ref{ex:ARCH}. 
To summarize, our theory for lasso in high-dimensional time series estimation extends beyond the classical linear Gaussian settings and provides guarantees potentially in the presence of model mis-specification, subgaussian innovations and/or nonlinearity in the DGM.


\subsection{Recent Work on High Dimensional Time Series}
\label{sec:worksummary}

Because our predictive model is the VAR, we wish to mention that recently, \cite{basu2015regularized} took a step forward in providing guarantees for lasso in finite lag Gaussian VAR models(see Example~\ref{ex:GaussVAR}) 
 in terms of their measure of stability. 
Their bounds are more general than the previous work~\citep{negahban2011estimation,loh2012high,han2013transition} by lifting the operator norm bound condition on the transition matrix. These operator norm conditions are restrictive even for VAR models with a lag of $1$ and never hold if the lag is strictly larger than 1! Therefore, the results of \cite{basu2015regularized} are very interesting. But they do have limitations.

A key limitation is that \cite{basu2015regularized} assumes that the VAR model is the true DGM which is critical in their analysis. The VAR model assumption, though popular, can be restrictive. For instance, the VAR family is not closed under linear transformations: if $Z_t$ is a VAR process then
$C Z_t$ may not expressible as a finite lag VAR~\citep{lutkepohl2005new}. In Section~\ref{sect:exmp}, we provide an example (Example~\ref{ex:sGmisVAR}) of VAR processes where omitting a single variable breaks down the VAR assumption.

Many authors have contributed to the high-dimensional time series literature. We include a representative sample here. On the applied side, \cite{chudik2011infinite,chudik2013econometric,chudik2014theory} use high-dimensional time series for global macroeconomic modeling. Methodological advances on high-dimensional time series estimation abound in the last decade. Although Lasso retains a significant presence, alternatives to lasso have been explored including quantile based methods for heavy-tailed data~\citep{qiu2015robust}, quasi-likelihood approaches~\citep{uematsu2015penalized}, two-stage estimation techniques~\citep{davis2012sparse} and the Dantzig selector~\citep{han2013transition, han2015direct}.  

Various authors have investigated the theoretical aspects of the topic with their own sets of assumptions on the underlying DGMs. Some of the earlier work (\cite{song2011large}, \cite{wu2015high} and \cite{alquier2011sparsity}) gave theoretical lasso guarantees assuming that RE conditions hold. However, as~\cite{basu2015regularized} pointed out, it is non-trivial to actually establish RE conditions in the presence of dependence. 
Both \cite{han2013transition} and \cite{han2015direct} studied the stable  Gaussian VAR models while this paper covers wider classes of processes as our examples demonstrate. { \cite{fan2016penalized} considered the case of multiple sequences of univariate $ \alpha $-mixing heavy-tailed dependent data. Under a stringent condition on the auto-covariance structure (please refer to Appendix D in~\cite{wong2017lasso} for details), the paper established finite sample $ \ell_2 $ consistency in the real support for penalized least squares estimators. In addition, under mutual incoherence type assumption,
it provided sign and $ \ell_\infty $ consistency. An AR(1) example was given as an illustration.} 
Both~\cite{uematsu2015penalized} and~\cite{kock2015oracle} establish oracle inequalities for  lasso applied to time series prediction. \cite{uematsu2015penalized} provided results not just for lasso but also for estimators using penalties such as the SCAD penalty. Also, instead of assuming Gaussian errors, it assumed only that fourth moments of the errors exist. \cite{kock2015oracle} provided non-asymptotic lasso error and prediction error bounds for stable Gaussian VARs. Both \cite{sivakumar2015beyond} and \cite{medeiros2016} considered subexponential designs. \cite{sivakumar2015beyond} studied lasso on iid subexponential designs and provide finite sample bounds. \cite{medeiros2016} studied adaptive lasso for linear time series models and provided sign consistency results.
\cite{wang2007regression} provided theoretical guarantees for lasso in linear regression models with autoregressive errors.
Other structured penalties beyond the $\ell_1$ penalty have also been considered \citep{nicholson2014hierarchical,nicholson2015varx,guo2015high,ngueyep2014large}.
\cite{zhang2015gaussian},~\cite{mcmurry2015high},~\cite{wang2013sparse}
and~\cite{chen2013covariance} consider estimation of the covariance (or precision) matrix of high-dimensional time series.
\cite{mcmurry2015high} and~\cite{nardi2011autoregressive} both highlight that autoregressive (AR) estimation, even in univariate time series, leads to high-dimensional parameter estimation problems if the lag is allowed to be unbounded.

\section{Preliminaries}\label{section:prelim}

\paragraph{Lasso Procedure for Dependent Data}
We describe our lasso procedure for estimation in dependent data. 
Given a stationary stochastic process of pairs $ (X_t, Y_t)_{t=1}^\infty $ where
$ X_t\in \R^p ,\, Y_t\in \R^q,\, \forall t $, we are interested in predicting $ Y_t$ given $ X_t $. In particular, given a dependent sequence  $(Z_t)_{t=1}^T$, one might want to forecast the present $Z_t$ using the past $ (Z_{t-d},\ldots,Z_{t-1})$. A linear predictor is a natural choice for that purpose. To put it in the regression setting, we identify  $Y_t = Z_t$ and $X_t = (Z_{t-d},\ldots,Z_{t-1})$. The pairs $(X_t, Y_t)$ defined as such are no longer iid. Assuming strict stationarity, the parameter matrix of interest $\bstar \in \R^{p \times q} $ is minimizer of the mean squared error loss
\begin{equation} \label{eqn:bstar}
\bstar = \argmin_{\bb \in \R^{p \times q}} \E [ \vertii{ Y_t - \bb' X_t}_2^2  ] .
\end{equation}

Note that $\bstar$ is independent of $t$ owing to stationarity. Because of high dimensionality ($ pq \gg T $), consistent estimation is impossible without regularization. We consider the lasso procedure. The  $\ell_1$-penalized least squares estimator $\bhat \in \R^{p \times q} $ is defined as
\begin{equation} \label{eqn:bhat}
\bhat = \argmin_{\bb \in \R^{p \times q}} \frac{1}{T}\Vert \vect( \vc{Y}-\mt{X}\bb ) \Vert_2^2+ \lambda_T \vertii{\vect(\bb)}_1  .
\end{equation}
where
\begin{align}\label{eq:XYdef}
\vc{Y} &=(Y_1,Y_2,\,\ldots\,, Y_T)' \in \R ^{T \times q} 
&
\mt{X} &=(X_1,X_2,\,\ldots\,, X_T)' \in \R ^{T \times p} .
\end{align}
The following matrix of true residuals is not available to an estimator but will appear in our analysis:
\begin{align}\label{eq:Wdef}
\mt{W} &:= \mt{Y} - \mt{X} \bstar.
\end{align}


\paragraph{Matrix and Vector Notation}
For a symmetric matrix $ \mt{M} $, let $\lmax{\mt{M}}$ and $\lmin{\mt{M}}$ denote its maximum and minimum eigenvalues respectively. For any matrix let $\mt{M}$, $\sr{\mt{M}}$, $\vertiii{\mt{M}}$, $\vertiii{\mt{M}}_\infty$, and $\vertiii{\mt{M}}_F$ denote its spectral radius $\max_i{\{|\lambda_i(\mt{M})|\}}$, operator norm $\sqrt{\lambda_{\max}(\mt{M}'\mt{M})}$, entrywise $\ell_\infty$ norm $\max_{i,j} |\mt{M}_{i,j}|$, and Frobenius norm $\sqrt{\mathrm{tr}(\mt{M}'\mt{M})}$ respectively. For any vector $v\in \R^p$, $\vertii{v}_q$ denotes its $\ell_q$ norm $ (\sum_{i=1}^p |v_i|^q)^{1/q}$. Unless otherwise specified, we shall use $\vertii{\cdot}$ to denote the $\ell_2$ norm. For any vector $\vc{v} \in \R^p $, we use $\vertii{\vc{v}}_0$ and  $\vertii{\vc{v}}_{\infty}$ to denote $\sum_{i=1}^{p} \mathbbm{1} \{\vc{v}_i\neq 0\}$ and $\max_i\{|\vc{v}_i|\}$ respectively. Similarly, for any matrix $ \mt{M}$, $ \vertiii{\mt{M}}_{0}= \vertii{\vect(\mt{M})}_0$ where $\vect(\mt{M})$ is the vector obtained from $\mt{M}$ by concatenating the rows of $M$. We say that matrix $ \mt{M} $ (resp. vector $ \vc{v} $) is \textit{$ s $-sparse} if $ \vertiii{\mt{M}}_0=s$ (resp. $ \vertii{ \vc{v}}_0 =s $). We use $ \vc{v}' $ and $ \mt{M}' $ to denote the transposes of  $ \vc{v} $ and $ \mt{M} $ respectively. 
When we index a matrix, we adopt the following conventions. For any matrix $ \mt{M}\in \R^{p\times q} $, for $ 1\le i \le p$, $1\le j\le q $, we define  $  \mt{M}[i,j]\equiv\mt{M}_{ij}:=\vc{e}_i'\mt{M}\vc{e}_j $, $\mt{M}[i,:]\equiv\mt{M}_{i:}:=\vc{e}_i'\mt{M} $ and $ \mt{M}[:,j]\equiv\mt{M}_{:j}:=\mt{M}\vc{e}_j $ where $ \vc{e}_i  $  is the vector with all $ 0 $s except for a $ 1 $ in the $ i $th coordinate.
The set of integers is denoted by $\mathbb{Z}$.


For a lag $ l \in \mathbb{Z}$, we define the auto-covariance matrix w.r.t. $ (X_t, Y_t)_t $ as $\Sigma(l) = \Sigma_{({X;Y})}(l):=\E [(X_t;Y_t)(X_{t+l};Y_{t+l})'] $. Note that $\Sigma(-l) = \Sigma(l)'$. Similarly, the auto-covariance matrix of lag $ l $ w.r.t. $(X_t)_t$ is $ \Sigma_{{X}}(l):=\E[ X_tX_{t+l}']$, and w.r.t. $(Y_t)_t$ is $ \Sigma_{Y}(l):=\E [Y_t Y_{t+l}' ]$. The cross-covariance matrix at lag $l$ is $ \Sigma_{X,Y}(l):=\E[ X_t Y_{t+l}' ]$. Note the difference between $\Sigma_{(X;Y)}(l)$ and $\Sigma_{X,Y}(l)$: the former is a $(p+q) \times (p+q)$ matrix, the latter is a $p \times q$ matrix. 
Thus, $ \Sigma_{(X;Y)}(l)$ is a matrix consisting of four sub-matrices. Using Matlab-like notation, $\Sigma_{(X;Y)}(l)=[\Sigma_{X},  \Sigma_{X,Y};  \Sigma_{Y,X}, \Sigma_{Y} ] $.
As per our convention, at lag $ 0 $, we omit the lag argument $ l $. For example, $ \Sigma_{X,Y}$ denotes $\Sigma_{X,Y}(0) = \E[ X_t Y_t' ]$.

\paragraph{A Brief Introduction to the $ \beta $-Mixing Condition}\label{sec:mixingintro}
There are various approaches to  quantity and control dependence across observations in a stationary time series. Popular ones include physical and predictive dependence measures~\citep{wu2005nonlinear}, spectral analysis~\citep{basu2015regularized,priestley1981spectral,stoica1997introduction} and mixing coefficients~\citep{bradley2005basic}. We opt for the $ \beta $-mixing coefficients route in this paper because the $ \beta $-mixing coefficients of a process are preserved under measurable transformations (please see Fact~\ref{fact:mixingEquiv} for details) and at the same time, many interesting processes such as Markov and hidden Markov processes satisfy a $\beta$-mixing condition~\citep[Sec. 3.5]{vidyasagar2003learning}.

Mixing conditions~\citep{bradley2005basic} are well established in the stochastic processes literature as a way to allow for dependence in extending results from the iid case. The general idea is to first define a measure of dependence between two random variables $X,Y$ (that can vector-valued or even take values in a Banach space) with associated sigma algebras $\sigma(X), \sigma(Y)$. In particular, 
\begin{align*}
\beta(X,Y) &= \sup \frac{1}{2} \sum_{i=1}^I \sum_{j=1}^J | P(A_i \cap B_j) - P(A_i)P(B_j) | 
\end{align*}
where the last supremum is over all pairs of partitions $\{A_1,\ldots,A_I\}$ and $\{B_1,\ldots,B_I\}$ of the sample space $\Omega$ such that $A_i \in \sigma(X), B_j \in \sigma(Y)$ for all $i,j$.
Then for a stationary stochastic process $(X_t)_{t=-\infty}^{\infty}$, one defines the mixing coefficients, for $l \ge 1$,
\[
\beta(l) = \beta(X_{-\infty:t}, X_{t+l:\infty}) .
\]
The $\beta$-mixing condition has been of interest in statistical learning theory for obtaining finite sample generalization error bounds for empirical risk minimization~\citep[Sec. 3.4]{vidyasagar2003learning} and boosting~\citep{kulkarni2005convergence}
for dependent samples. There is also work on estimating $\beta$-mixing coefficients from data~\citep{mcdonald2011estimating}. 
Before we continue, we note
an elementary but useful fact about mixing conditions, viz. they persist under arbitrary measurable transformations of the original stochastic process.

\begin{fact}\label{fact:mixingEquiv}
Consider any $\beta$-mixing stationary process $ (U_t)_{t=1}^T $. Then, for any measurable function $ f(\cdot)$, the  stationary sequence of the transformed observations $ (f(U_t))_{t=1}^T $ is also $ \beta $-mixing in the same sense with its mixing coefficients bounded by those of the original sequence. 
\end{fact}

\section{Main Results}\label{sect:subgauss}

We start with  introducing two well-known sufficient conditions that enable us to provide non-asymptotic guarantees for lasso estimation and prediction errors -- the restricted eigenvalue (RE) and the deviation bound (DB) conditions.
The bulk of the technical work in this paper boils down to establishing, with high probability, that the RE and DB conditions hold under the subgaussian $\beta$-mixing assumptions (Propositions~\ref{results:REsub} and~\ref{result:betaDev}). 
In the classical linear model setting (see, e.g., Chap. 2.3 in~\cite{hayashi2000econometrics}) where sample size is larger than the dimensions ($n>p$), the conditions for consistency of the ordinary least squares (OLS) estimator are as follows:
(a) the empirical covariance matrix $\mt{X}'\mt{X}/T \overset{P}{\rightarrow}Q$ and $ Q $ invertible, i.e., $\lmin{Q}>0$, and (b)
the regressors and the noise are asymptotically uncorrelated, i.e., $\mt{X}'\mt{W} /T\rightarrow \vc{0}$.

In high-dimensional regimes,~\cite{bickel2009simultaneous},~\cite{loh2012high} and~\cite{negahban2012restricted} have established similar consistency conditions for lasso. The first one is the \textit{restricted eigenvalue} ({RE}) condition on $\mt{X}'\mt{X}/T$ (which is a special case, when the loss function is the squared loss, of the \textit{restricted strong convexity} ({RSC}) condition). The second is the \textit{deviation bound} ({DB}) condition on $\mt{X}'\mt{W}$.
The following lower    {RE} and    {DB} definitions are modified from those
given by \cite{loh2012high}. 

\begin{defn}[Lower Restricted Eigenvalue]\label{defn:RE}
A symmetric matrix ${\Gamma}\in \R^{p\times p} $ satisfies a lower restricted eigenvalue condition with curvature $\alpha>0$ and tolerance $\tau(T,p)>0$ if
$$
\forall \vc{v} \in \R^{p},\ \vc{v}' {\Gamma}\vc{v} \ge \alpha \vertii{\vc{v}}^2_2 - \tau(T,p)\vertii{\vc{v}}^2_1 .
$$
\end{defn}

\begin{defn}[Deviation Bound]\label{defn:DB}
Consider the random matrices $\mt{X} \in \R^{T\times p}$ and $\mt{W}\in \R^{T\times q}$ defined in~\eqref{eq:XYdef} and~\eqref{eq:Wdef} above. They are said to satisfy the deviation bound condition if there exist a deterministic multiplier function $ \mathbb{Q}(\mt{X},\mt{W},\bstar)$ and a rate of decay function $\mathbb{R}(p,q,T)$ such that:
$$
\frac{1}{T}\vertiii{\mt{X}'\mt{W}}_{\infty} \le \mathbb{Q}(\mt{X},\mt{W},\bstar) \mathbb{R}(p,q,T) .
$$
\end{defn}


%

We will show that, with high probability, the RE and DB conditions hold for dependent data that satisfy Asumptions \ref{as:spars}--\ref{assum:beta} described below. 
We shall do that \emph{without} assuming any parametric form of the data generating mechanism. Instead, we will assume a subgaussian tail condition on the random vectors $X_t,Y_t$ and that they satisfy the geometrically $ \beta $-mixing condition. 

\subsection{Assumptions}
\begin{assump}[Sparsity]\label{as:sparse}
The matrix $\bstar$ is $s$-sparse, i.e. $\vertii{\vect(\bstar)}_0\le s$. \label{as:spars}
\end{assump}
\begin{assump}[Stationarity]\label{as:stat}
The process $(X_t, Y_t)$ is strictly stationary: i.e., $ \forall t, \tau,\, n \ge 0$,
\[
((X_{t},Y_{t}),\cdots ,(X_{t+n},Y_{t+n}))~\overset{d}{=} ~((X_{t+\tau},Y_{t+\tau}),\cdots,(X_{t+\tau+n},Y_{t+\tau+n})) .
\] where ``$\overset{d}{=}$'' denotes equality in distribution. \label{as:stat}
\end{assump}
\begin{assump}[Centering]\label{as:0mean}
We have, $\forall t,\ \E(X_t)=\vc{0}_{p \times 1}, $
and
$\E(Y_t)=\vc{0}_{q \times 1}$ .
\end{assump}


The thin tail property of the Gaussian distribution is desirable from the theoretical perspective, so we would like to keep that but at the same time allow for more generality. The subgaussian distributions are a nice family characterized by having tail probabilities of the same as or lower order than the Gaussian.
We now focus on subgaussian random vectors and present high probabilistic error bounds with all parameter dependences explicit. 

\begin{assump}[Subgaussianity]\label{assum:subgauss}
The subgaussian constants of $X_t$ and $Y_t$ are bounded above by $\sqrt{K_X}$ and $\sqrt{K_Y}$ respectively. (Please see Appendix~\ref{appendix:sgDef} for a detailed introduction to subgaussian random vectors. )
\end{assump}

Classically, mixing conditions were introduced to generalize classic limit theorems in probability beyond the case of iid random variables~\citep{rosenblatt1956central}. 

\begin{assump}[$\beta$-Mixing]\label{assum:beta}
The process $\left((X_t, Y_t)\right)_{t}$ is geometrically $\beta$-mixing, i.e., there exists some constant $ c_{\beta}>0 $ such that $ \forall l\ge 1,\,\beta(l) \allowdisplaybreaks \le \allowdisplaybreaks \exp(-c_\beta l),\, $

\end{assump}

The $\beta$-mixing condition allows us to apply the independent block technique developed by \cite{yu1994rates}. For examples of large classes of Markov and hidden Markov processes that are geometrically
$\beta$-mixing, see Theorem 3.11 and Theorem 3.12 of \cite{vidyasagar2003learning}.
In the independent blocking technique, we construct a new set of \emph{independent} blocks such that each block has the same distribution as that of the corresponding block from the original sequence. Results of~\cite{yu1994rates} provide upper bounds on the difference between probabilities of events defined using the independent blocks versus the same event defined using the original data. Classical probability theory tools for independent data can then be applied on the constructed independent blocks. In Appendix~\ref{sec:subgaussproofs}, we apply the independent blocking technique to Bernstein's inequality to get the following concentration inequality for $\beta$-mixing random variables.

\begin{lm}[Concentration of $\beta$-Mixing Subgaussian Random Variables] \label{result: subgau:tailbdd}
Let ${Z} = (Z_1,\ldots,Z_T)$ consist of a sequence of mean-zero random variables with exponentially decaying $\beta$-mixing coefficients as in~\ref{assum:beta}. Let $K$ be such that $\max_{t=1}^T \snorm{Z_t} \le \sqrt{K}$. Choose a block length $a_T \ge 1$ and let $\mu_T = \lfloor T/(2a_T) \rfloor$. We have, for any $t>0 $,
\begin{align*}
\mathbb{P}[ \frac{1}{T} \vert  \Vert{Z}\Vert^2_2 - \mathbb{E}[\Vert {Z}\Vert^2_2] \vert >t ]
\le  &
4 \exp \bpar{  -C_B\min\bcur{ \frac{t^2 \mu_T}{K^2}, \frac{t \mu_T}{K}  }   } \\
  &+ 2(\mu_{T}-1) \exp\left( -c_{\beta} a_T\right)
  +
    \exp\left(\frac{- 2 t \mu_T}{K}  \right).
\end{align*}
In particular, for $0<t<K$,
\begin{align*}
\mathbb{P} \bbra{ \frac{1}{T} \vert  \Vert {Z}\Vert^2_2 - \mathbb{E}[\Vert {Z}\Vert^2_2]  \vert 
>
t
}  
\le &
4 \exp\bpar{ -C_B \frac{t^2 \mu_T}{K^2}} \\
&+
2(\mu_T-1) \exp\left( -c_{\beta} a_T\right)
+
\exp\left(\frac{- 2 t \mu_T}{K}  \right).
\end{align*}
Here $C_B$ is the universal constant appearing in Bernstein's inequality (Proposition~\ref{thm:Bernstein}).
\end{lm}

\begin{rem}
The three terms in the bound above all have interpretations: the first is a concentration term with a rate that depends on the ``effective sample size'' $\mu_T$, the number of blocks; the second
is a dependence penalty accounting for the fact that the blocks are not exactly independent; and the third is a remainder term coming from the fact that $2a_T$ may not exactly divide $T$. The key terms are the first two
and exhibit a natural trade-off: increasing $a_T$ worsens the first term since $\mu_T$ decreases, but it improves the second term since there is less dependence at larger lags.
\end{rem}

\subsection{High Probability Guarantees for the Lower Restricted Eigenvalue and Deviation Bound Conditions}

We show that both lower RE and DB conditions hold, with high probability, under our assumptions.

\begin{pr}[RE]\label{results:REsub}
Suppose Assumptions~\ref{as:sparse}--\ref{assum:beta} hold.
Let $C_B$ be the Berstein's inequality constant,
$\csub=\min\{C_B, 2\}$,
$b = \min\{\tfrac{1}{54K_X}\lmin{\Sigma_X},1  \}$ and
$c = \tfrac{1}{6}\max\{c_{\beta},\csub b^2\}$.
Then for
{$T \ge  \bpar{\frac{1}{c}\log(p)}^{2}$}, with probability at least $
1- 
  5\exp \bpar{  -\csub T^{\frac{1}{2}}  }
  -
  2(T^{\frac{1}{2}}-1)\exp \bpar{
  -c_{\beta}T^{\frac{1}{2}} 
    },
$ we have for every vector $\vc{v} \in \R^p$,
$$
\vc{v}'  \hat{\Gamma}   \vc{v}
\ge
\alpha_2\vertii{\vc{v}}^2 
- 
\tau_2(T,p)
\vertii{\vc{v}}_1^2 ,
$$
where
$\alpha_2 =\half \lmin{\Sigma_X}$ , and
$\tau_2(T,p) ={27bK_X\log(p)}/{c  T^\half}$ .
\end{pr}

\begin{pr}[Deviation Bound]\label{result:betaDev}
Suppose Assumptions \ref{as:sparse}--\ref{assum:beta} hold.
Let $K = \sqrt{K_Y} + \sqrt{K_X}\bpar{1+\vertiii{\bstar} }$ and $\xi \in (0,1)$ be a free parameter. Then, for sample size
$$
T\ge \max \bcur{
\bpar{ \log(pq)  \max\bcur{ \frac{K^4}{2C_B},K^2 } }^{\frac{1}{1-\xi}},
\bbra{\frac{2}{c_{\beta}}\log(pq)   }^\frac{1}{\xi} 
},
$$
we have
\begin{align*}
\mathbb{P} &
\bbra{ \frac{1}{T} \vertiii{\mt{X}'\mt{W}}_{\infty} 
\le 
 \mathbb{Q}(\mt{X},\mt{W},\bstar)  \mathbb{R}(p,q,T) 
} \\
&\quad \quad \ge  1 -
15\exp\bpar{ - \half \log(pq)} 
 -
6(T^{1-\xi}-1) \exp\bpar{-\half c_{\beta}T^{\xi}  }
\end{align*}
where
\begin{align*}
 \mathbb{Q}(\mt{X},\mt{W},\bstar)  &=\sqrt{\frac{2K^4}{C_B}}  ,
 &
 \mathbb{R}(p,q,T) &= \sqrt{\frac{\log(pq)}{T^{1-\xi}}} .
\end{align*}
\end{pr}

\begin{rem}
Since $\xi\in (0,1)$ is a free parameter, we choose it to be arbitrarily close to zero so that $\mathbb{R}(p,q,T)$ scales at a rate arbitrarily close to $\sqrt{\frac{\log(pq)}{T}}$. However, there is a price to pay
for this: both the initial sample threshold and the success probability worsen as we make $\xi$ very small.
\end{rem}

\subsection{Estimation and Prediction Errors}
The guarantees below follow easily from plugging the RE and DB constants from Propositions~\ref{results:REsub} and~\ref{result:betaDev} into a ``master theorem'' (Theorem~\ref{result:master} in Appendix~\ref{sec:masterproof}). Similar results are well-known in the literature (e.g., see \cite{bickel2009simultaneous,loh2012high,negahban2012restricted}). 
The extra  generality here, which is critical for the analysis in this paper,  comes from allowing the response vector and regressors to potentially be in different dimensions and the object of estimation to be a matrix. 

\begin{cor}[Lasso Guarantee under Subgaussian Tails and $\beta$-Mixing]
\label{cor:subgauss}
Suppose Assumptions \ref{as:spars}--\ref{assum:beta} hold. Let $C_B,\csub,c,b$ and $K$ be as defined in Propositions \ref{results:REsub} and \ref{result:betaDev} and $ \tilde{C} := \min \{C, c_\beta\} $. Let $ \xi\in(0,1)$ be a free parameter. Then, for sample size
\begin{align*}
\begin{split}
T\ge &
\max 
\left\{
 \bpar{\frac{\log(p)}{c}}^2 \max\bcur{\bpar{ \frac{1728sbK_X }{\lmin{\Sigma_{{X}}}} }^2,1 } 
, \right. \\
&\quad \qquad\left.
\bpar{ \log(pq)  \max\bcur{ \frac{K^4}{2C_B},K^2 } }^{\frac{1}{1-\xi}},
\bbra{\frac{2}{c_{\beta}}\log(pq)   }^\frac{1}{\xi} 
\right\}
\end{split}
\end{align*}
we have with probability at least
\begin{align*}
1-
15\exp\bpar{ - \half \log(pq)} 
-
6(T^{1-\xi}-1) \exp\bpar{-\half c_{\beta}T^{\xi}  } 
  -
  5(T^{\frac{1}{2}}-1)\exp \bpar{
  -\tilde{C}T^{\frac{1}{2}}
    }
\end{align*}

the lasso estimation and (in-sample) prediction error bounds 
\begin{eqnarray}
\vertii{\vect(\bhat-\bstar)} \le 4\sqrt{s}\lambda_T/\alpha ,			\label{eq:l2errorBdd}
\\
\vertiii{ (\bhat-\bstar)' \hat{\Gamma} (\bhat-\bstar)  }_F^2
\le
\frac{32\lambda_T^2 s}{\alpha}			.							\label{eq:predErrBdd}
\end{eqnarray}
hold with
\begin{align*}
\alpha &=\half \lmin{\Sigma_X}  ,&
\lambda_T &= 4 \mathbb{Q}(\mt{X},\mt{W},\bstar) \mathbb{R}(p,q,T)
\end{align*}
where
\begin{align*}
\hat{\Gamma}  &:=\mt{X}'\mt{X}/T, 
&
 \mathbb{Q}(\mt{X},\mt{W},\bstar)  &=\sqrt{\frac{2K^4}{C_B}}  ,
 &
 \mathbb{R}(p,q,T) &= \sqrt{\frac{\log(pq)}{T^{1-\xi}}} .
\end{align*}
\end{cor}

\begin{rem}
The condition number of $ \Sigma_X $ plays an important part in the literature of lasso error guarantees~\citep[e.g.]{ loh2012high}. Here, we see that the role of the condition number $\lmax{\Sigma_X}/\lmin{\Sigma_X}$ is replaced by $K_X/\lmin{\Sigma_X}$ that now serves as the ``effective condition number."
\end{rem}

\section{Examples}\label{sect:exmp}

We explore applicability of our theory beyond just linear Gaussian processes using the examples below. In the following examples, we identify $ X_t := Z_t $ and $ Y_t :=Z_{t+1}$ for $ t=1,\ldots,T$. For the specific parameter matrix $ \bstar  $ in each Example below, we can verify  that Assumptions \ref{as:sparse}--\ref{assum:beta} hold (see Appendix~\ref{Apnx:Ver}) for details. Therefore, Propositions~\ref{results:REsub} and \ref{result:betaDev} and Corollary \ref{cor:subgauss} follow. Hence we have all the high probabilistic guarantees for lasso on data generated from DGM potentially involving subgaussianity, model mis-specification, and/or nonlinearity. 

\begin{exmp}[Gaussian VAR]\label{ex:GaussVAR} Transition matrix estimation in sparse stable VAR models has
been a popular topic in recent years~\citep{davis2015sparse,han2013transition,song2011large}.
%
The lasso estimator is a natural choice for the problem.

We state the following convenient fact because it allows us to study any finite order VAR model by considering its equivalent VAR($ 1 $) representation. See Appendix \ref{veri:VAR} for details.

\begin{fact}
Every VAR($d$) process can be written in VAR($ 1 $) form (see e.g. \cite[Ch 2.1]{lutkepohl2005new}).
\end{fact}
Therefore, without loss of generality, we can consider VAR($ 1 $) model in the ensuing Examples. 

Formally a first order Gaussian VAR($ 1 $) process is defined as follows.
Consider a sequence of serially ordered random vectors $(Z_t)$, ${Z_t}\in \R^p$ that admits the following auto-regressive representation:
\begin{align}\label{eq:VAR(1)}
{Z_t} = \mt{A} {Z}_{t-1} + {\mathcal{E}}_t
\end{align}
where  $\mt{A}$ is  a non-stochastic coefficient matrix in $\R^{p \times p}$ and innovations ${\mathcal{E} }_t$ are $p$-dimensional random vectors from $\mathcal{N}( \vc{0}, \Sigma_{\epsilon})$ with $\lmin{\Sigma_{\epsilon}}>0$ and $\lmax{\Sigma_{\epsilon}}< \infty$. 

Assume that the VAR($ 1 $) process is \emph{stable}; i.e. $ \mathrm{det}\bpar{\mt{I}_{p \times p}-\mt{A} z} \neq 0,\, \forall \verti{z} \le 1 $. Also, assume $ \mt{A} $ is $ s $-sparse. In here, $\bstar = \mt{A}' \in \R^{p \times p}$. 



%
%
\end{exmp}

\begin{exmp}[VAR with Subgaussian Innovations]\label{ex:sGVAR}

Consider a VAR($ 1 $) model defined as in Example \ref{ex:GaussVAR} except that we replace the Gaussian white noise innovations with subgaussian ones and assume $ \vertiii{\mt{A}}<1 $. 

For example, take iid random vectors from the uniform distribution; i.e. $\forall t,\,\mathcal{E}_t \overset{iid}{\sim} \U{\bbra{-\sqrt{3}, \sqrt{3}}^p} $. These $\mathcal{E}_t$ will be independent centered isotropic subgaussian random vectors,
giving us we a VAR($ 1 $) model with subgaussian innovations.  If we take a sequence $ (Z_t)_{t=1}^{T+1} $ generated according to the model, each element $ Z_t$ will be a mean zero subgaussian random vector. Note that $ \bstar =A' $.

\end{exmp}

\begin{exmp}[VAR with subgaussian Innovations and Omitted Variable]\label{ex:sGmisVAR}
We will study estimation of a VAR(1) process when there are endogenous variables omitted. This arises naturally when the underlying DGM is high-dimensional but not all variables are available (perhaps they are not observable or measurable) to the researcher to do estimation or prediction. This also happens when the researcher mis-specifies the scope of the model.

Notice that the system of the retained set of variables is no longer a finite order VAR (and thus non-Markovian). 
This example serves to illustrate that our theory is applicable to models beyond the finite order VAR setting. 

Consider a VAR(1) process $ (Z_t, \Xi_t)_{t=1}^{T+1} $ such that each vector in the sequence is generated by the recursion below:
$$
(Z_t; \Xi_t) = \mt{A} (Z_{t-1}; \Xi_{t-1}) + (\mathcal{E} _{Z, t-1}; \mathcal{E} _{\Xi, t-1})
$$
where $Z_t \in \R^{p} $,  $\Xi_t \in \R $, $ \mathcal{E} _{Z, t} \in \R^{p}  $, and $\mathcal{E} _{\Xi, t} \in \R $ are partitions of the random vectors $ (Z_t, \Xi_t) $ and $ \mathcal{E}_t $ into $ p $ and $ 1 $ variables. Also,
$$
\mt{A}:= 
\bbra{
\begin{array}{cc}
\mt{A}_{ZZ} & \mt{A}_{Z\Xi} \\ 
\mt{A}_{\Xi Z} & \mt{A}_{\Xi \Xi}
\end{array}  
}
$$
is the coefficient matrix of the VAR(1) process with $\mt{A}_{Z \Xi }   $ $ 1 $-sparse, $ \mt{A}_{Z Z } $  $ p $-sparse and $  \vertiii{\mt{A}} < 1$. $ \mathcal{E}_t := (\mathcal{E} _{X, t-1}; \mathcal{E} _{Z, t-1}) $ for $ t=1,\ldots,T+1$ are iid draws from a subgaussian distribution; in particular we consider the subgaussian distribution described in Example~\ref{ex:sGVAR}. 

We are interested in the OLS $ 1 $-lag estimator of the system restricted to the set of variables in $ Z_t $. Recall that
$$
\bstar := \argmin_{\mt{B} \in \R ^{p\times p}} \E \bpar{  
\vertii{
Z_t - \mt{B}'Z_{t-1}
}^2_2
}
$$
We show in the appendix that $ (\bstar)' =\mt{A}_{Z Z }+\mt{A}_{Z \Xi } \Sigma_{\Xi Z }(0)(\Sigma_Z )^{-1} $ is sparse.



%

\end{exmp}

\begin{exmp}[Multivariate ARCH]\label{ex:ARCH}
We will explore the generality of our theory by considering a multivariate nonlinear time series model with subgaussian innovations.  A popular nonlinear multivariate time series model in econometrics and finance is the vector autoregressive conditionally heteroscedastic (ARCH) model. We chose the following specific ARCH model for convenient validation of the geometric $ \beta $-mixing property; it may potentially be applicable to a larger class of multivariate ARCH models. Consider a sequence of random vector $ (Z_t)_{t=1}^{T+1} $ generated by the following recursion. For any constants $ c>0 $, $ m\in (0,1) $, $a > 0$, and $ \mt{A} $ sparse
with $\vertiii{\mt{A}} < 1$:
\begin{align} 
\begin{split} \label{eq:ARCH}
Z_t&= \mt{A} Z_{t-1}+\Sigma(Z_{t-1}) \mathcal{E}_t \\
\Sigma(\vc{z}) &:= c \cdot \clip{\vertii{\vc{z}}^{m}}{a}{b}\mt{I}_{p \times p}
\end{split}
\end{align}
where $ \mathcal{E}_t    $ are iid random vectors from some subgaussian distribution and $\clip{x}{a}{b}$ clips the argument $x$ to stay in the interval $[a,b]$. We can take innovations $ \mathcal{E}_t    $ to be iid random vectors from uniform distribution as described in Example~\ref{ex:sGVAR}.
Consequently, each $ Z_t $ will be a mean zero subgaussian random vector. Note that $ \bstar = \mt{A}'$, the transpose of the coefficient matrix $ \mt{A} $ here.

\end{exmp}

\section{Simulations}\label{sect:sim}
Corollary~\ref{cor:subgauss} in Section~\ref{sect:subgauss} makes a precise prediction for the $ \ell_2 $ parameter error $ \vertiii{\Theta^\ast - \hat{\Theta}}_{F} $.
We report scaling simulations for Examples 1--4 to confirm the sharpness of the bounds. 
 
Sparsity is always $ s = \sqrt{p} $, noise covariance matrix $ \Sigma_{\epsilon} = I_p $, and the operator norm of the driving matrix set to $ \vertiii{\mt{A}}= 0.9 $. The problem dimensions are $ p\in \{50, 100, 200, 300\} $.
Top left, top right, bottom left and bottom right sub-figures in Figure~\ref{fig:subGauss} correspond to simulations of Examples~\ref{ex:GaussVAR}, \ref{ex:sGVAR}, \ref{ex:sGmisVAR} and \ref{ex:ARCH} respectively.

In all combinations of the four dimensions and Examples, the error decreases to zero as the sample size $ n $ increases, showing consistency of the method. In each sub-figure, the $ \ell_2 $ parameter error curves align when plotted against a suitably rescaled sample size ($ \frac{T}{s \log(p)} $) for different values of dimension $ p $.   We see the error scaling agrees nicely with theoretical guarantees provided by Corollary~\ref{cor:subgauss}. 

\begin{figure}[htp]

\centering
\includegraphics[width=.5\textwidth]{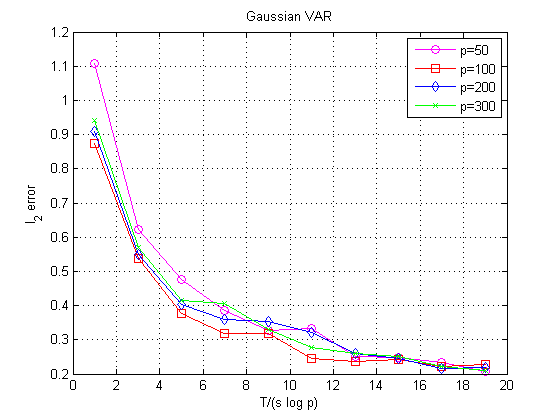}\hfill
\includegraphics[width=.5\textwidth]{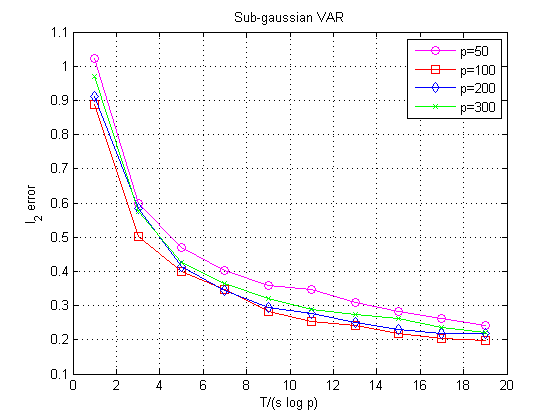}
\includegraphics[width=.5\textwidth]{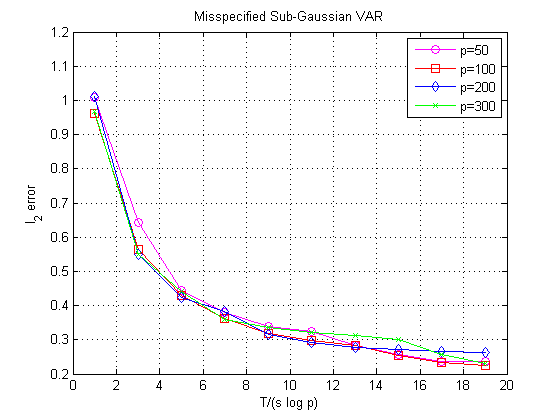}\hfill
\includegraphics[width=.5\textwidth]{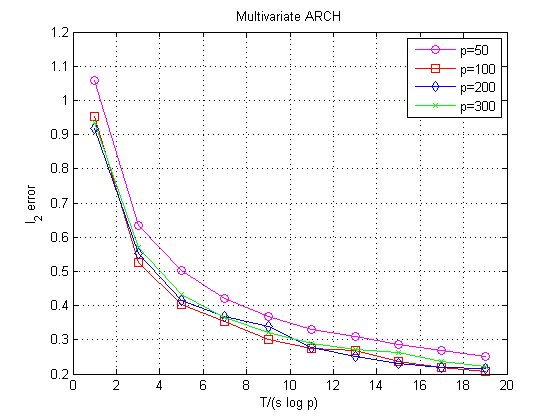}
\caption{$ \ell_2 $ estimation error of lasso against rescaled sample size for Examples 1--4. }
\label{fig:subGauss}

$  $\\

\end{figure}

%


\clearpage
\appendix

\section{Subgaussian Constants for Random Vectors }\label{appendix:sgDef}

The subgaussian and subexponential constants have various equivalent definitions, we adopt the following from \cite{rudelson2013hanson}. 

\begin{defn}[Subgaussian Norm and Random Variables/Vectors]
A random variable $ U $ is called subgaussian with subgaussian constant $K$ if its subgaussian norm
$$
\snorm{U}  := \sup_{p \ge 1} p^{-\half} (\E\verti{U}^p)^{1/p}
$$
satisfies $\snorm{U} \le K$.

A random vector $ V \in \R^n  $ is called subgaussian if all of its one-dimensional projections are subgaussian and we define 
$$ \snorm{{V}}:= \sup_{\vc{v}\in \R^n:\vertii{\vc{v}}\le 1}\snorm{\vc{v}'{V}} $$.
\end{defn}

\begin{defn}[Subexponential Norm and Random Variables/Vectors]
A random variable $ U $ is called subexponential with subexponential constant $K$ if its subexponential norm
$$
\enorm{U}  := \sup_{p \ge 1} p^{-1} (\E\verti{U}^p)^{1/p}
$$
satisfies 
$\enorm{U} $ $\le K$.

A random vector $ V\in \R^n  $ is called subexponential if all of its one-dimensional projections are subexponential and we define 
$$ \enorm{{U}}:= \allowbreak \sup_{\vc{v}\in \allowbreak \R^n :\allowbreak\vertii{\vc{v}}\le 1}\enorm{\vc{v}'{V}} $$
\end{defn}

\begin{fact}\label{fact:subgaussExp}
A random variable $ U $ is subgaussian iff $ U^2 $ is subexponential with $ \snorm{U}^2 = \enorm{U^2} $.
\end{fact}

\section{Proof of Master Theorem}\label{sec:masterproof}
We present a master theorem that provides guarantees for the $\ell_2$ parameter estimation error and for the (in-sample) prediction error. The proof builds on existing result of the same kind~\citep{bickel2009simultaneous,loh2012high,negahban2012restricted} and we make no claims of originality for either the result or for the proof. 

\begin{thm}[Estimation and Prediction Errors] \label{result:master}
Consider the lasso estimator $\bhat$ defined in \eqref{eqn:bhat}. Suppose Assumption~\ref{as:spars} holds. Further, suppose that $\hat{\Gamma}  :=\mt{X}'\mt{X}/T$ satisfies the lower RE$(\alpha, \tau)$ condition with $\alpha \ge 32s\tau$ and  $\mt{X}'\mt{W}$ satisfies the deviation bound. Then, for any  $\lambda_T\ge 4 \mathbb{Q}(\mt{X},\mt{W},\bstar)\mathbb{R}(p,q,T) $, we have the following guarantees:
\begin{eqnarray}
\vertii{\vect(\bhat-\bstar)} \le 4\sqrt{s}\lambda_T/\alpha ,			\label{eq:l2errorBdd}
\\
\vertiii{ (\bhat-\bstar)' \hat{\Gamma} (\bhat-\bstar)  }_F^2
\le
\frac{32\lambda_T^2 s}{\alpha}			.							\label{eq:predErrBdd}
\end{eqnarray}
\end{thm} 
\begin{proof}[Proof of Theorem \ref{result:master}]

We wil break down the proof in steps.

\begin{enumerate}
\item
Since $\bhat$ is optimal for  \ref{eqn:bhat} and $\bstar$ is feasible,
$$
\frac{1}{T} \vertiii{ \vc{Y}-\mt{X}\bhat}_F^2 + \lambda_T \vertii{\vect(\bhat)}_1    \le  
\frac{1}{T} \vertiii{ \vc{Y}-\mt{X}\bstar}_F^2 + \lambda_T \vertii{\vect(\bstar)}_1  
 $$

\item \label{main:2}
Let $\hat{\Delta}:=\bhat-\bstar \in \R^{p\times q}  $ 
$$
\frac{1}{T}\vertiii{\mt{X}\hat{\Delta}}_F^2 \le \frac{2}{T} \mathrm{tr}(\hat{\Delta}' \mt{X}'\mt{W} ) +\lambda_T\left( \vertii{\vect(\bstar)}_1 - \vertii{\vect(\bhat)}_1\right)
$$
Note that
\begin{align*}
\vertii{\vect(\bstar+\hat{\Delta})}_1 -\vertii{\vect(\bstar)}_1 \ge &
\{\vertii{\vect(\bstar_S)}_1 - \vertii{\vect(\hat{\Delta}_S)}_1 \} \\
&+ \vertii{\vect(\hat{\Delta}_{S^c})}_1- \vertii{\vect(\bstar)}_1 \\
&=\vertii{\vect(\hat{\Delta}_{S^c})}_1 - \vertii{\vect(\hat{\Delta}_{S})}_1
\end{align*}
where $S$ denote the support of $\bstar$.
\item \label{main:3}
With $RE$ constant $\alpha$ and tolerance $\tau$, deviation bound constant $\mathbb{Q}(\Sigma_X, \Sigma_W)$ and $\lambda_T\ge 2\mathbb{Q}(\Sigma_X, \Sigma_W)\sqrt{\frac{\log(q)}{T}} $, we have 
\begin{align*}
\alpha \vertiii{\hat{\Delta}}_F^2  -&\tau\Vert \vect(\hat{\Delta} )\Vert_1^2 
\\ &\overset{RE}{\le} 
\frac{1}{T}\vertiii{\mt{X}\Delta }_F^2 
\\&\le 
\frac{2}{T}\mathrm{tr}(\hat{\Delta}' \mt{X}'\mt{W})+  \lambda_T\{\vertii{\vect(\hat{\Delta}_{S})}_1 - \vertii{\vect(\hat{\Delta}_{S^c})}_1\}
 \\&\le 
\frac{2}{T}\sum_{k=1}^{q}\Vert\hat{\Delta}_{:k} \Vert_1 \Vert( \mt{X}'\mt{W})_{:k} \Vert_{\infty} + \lambda_T\{\vertii{\vect(\hat{\Delta}_{S})}_1 - \vertii{\vect(\hat{\Delta}_{S^c})}_1\}
 \\&\le
\frac{2}{T}\Vert\vect(\hat{\Delta}) \Vert_1 \vertiii{\mt{X}'\mt{W} }_{\infty} + \lambda_T\{\vertii{\vect(\hat{\Delta}_{S})}_1 - \vertii{\vect(\hat{\Delta}_{S^c})}_1\}
\\&\overset{DB}{\le}
2\Vert\vect(\hat{\Delta}) \Vert_1\mathbb{Q}(\Sigma_X, \Sigma_W)\mathbb{R}(p,q,T) + \lambda_T\{\vertii{\vect(\hat{\Delta}_{S})}_1 - \vertii{\vect(\hat{\Delta}_{S^c})}_1\}
\\&\le
\Vert\vect(\hat{\Delta}) \Vert_1\lambda_N / 2+ \lambda_T\{\vertii{\vect(\hat{\Delta}_{S})}_1 - \vertii{\vect(\hat{\Delta}_{S^c})}_1\}
\\&\le
\frac{3\lambda_T}{2}\vertii{\vect(\hat{\Delta}_{S})}_1 - \frac{\lambda_T}{2}\vertii{\vect(\hat{\Delta}_{S^c})}_1
\\&\le
2\lambda_T \vertii{\vect(\hat{\Delta})}_1
\end{align*}

\item  \label{main:4}
In particular, this says that 
$
3 \vertii{\vect(\hat{\Delta}_{S})}_1 \ge \vertii{\vect(\hat{\Delta}_{S^c})}_1
$\\
So $\vertii{\vect(\hat{\Delta})}_1\le 4 \vertii{\vect(\hat{\Delta}_S)}_1 \le 4\sqrt{s}\vertii{\vect(\hat{\Delta})}$
\item
Finally, with $\alpha \ge 32s\tau$,
\begin{align*}
\frac{\alpha	}{2}\vertii{\vect(\hat{\Delta} )}_F^2 &\le
( \alpha-16s\tau)\vertii{\vect(\hat{\Delta} )}_F^2  \\
&\le  \alpha \vertii{\vect(\hat{\Delta} )}_F^2  -\tau\Vert \vect(\hat{\Delta} )\Vert_1^2 \\
 &\le 2\lambda_T\Vert \vect(\hat{\Delta} )\Vert_1 \\
 &\le 2\sqrt{s}\lambda_T\Vert \hat{\Delta} \Vert_F
\end{align*}
\item \label{main:6}
$$
\vertii{\vect(\hat{\Delta} )}_F \le \frac{4\lambda_T \sqrt{s}}{\alpha}
$$
\item 

From step \ref{main:4}, we have 
\begin{align*}
\frac{1}{T}\vertiii{\mt{X}\hat{\Delta}}_F^2 
\le
8\lambda_T \sqrt{s} \vertii{\vect(\hat{\Delta})}
\end{align*}
Then, from step \ref{main:6}
\begin{align*}
\frac{1}{T}\vertiii{\mt{X}\hat{\Delta}}_F^2 
\le
8\lambda_T \sqrt{s} \vertii{\vect(\hat{\Delta})}
\le
32\lambda_T^2 s /\alpha
\end{align*}

\end{enumerate}
\end{proof}

\section{Proofs for Subgaussian Random Vectors under $\beta$-Mixing}\label{sec:subgaussproofs}

\begin{proof}[Proof of Lemma \ref{result: subgau:tailbdd}]
$ $\\
Following the description in \cite{yu1994rates}, we divide the stationary sequence of real valued random variables $ \{Z_t\}_{t=1}^T $ into $ 2\mu_T $ blocks of size $ a_T $ with a remainder block of length $ T-2\mu_T a_T $. Let $ H $ and $ T $ be sets that denote the indices in the odd and even blocks respectively, and let $ Re $ to denote the indices in the remainder block. To be specific,

\[
O= \cup_{j=1}^{\mu_T}O_j \;\; \text{where } O_j:= \{ i: 2(j-1)a_T +1 \le i \le  ( 2j-1)a_T\},\, \forall j
\]
\[
E:= \cup_{j=1}^{\mu_T}E_j \;\; \text{where } E_j:= \{ i: ( 2j-1)a_T +1  \le i \le  ( 2j)a_T\} ,\, \forall j
\]

Let $ \vc{Z}_o:=\{Z_t: t\in O\} $ be a collection of the random vectors in the odd blocks. Similarly, $ \vc{Z}_e:=\{Z_t: t\in E\} $ is a collection of the random vectors in the even blocks, and $ \vc{Z}_r:=\{Z_t: t\in Re\} $ a collection of the random vectors in the remainder block. Lastly, $ \vc{Z}:= \vc{Z}_O\cup \vc{Z}_e \cup \vc{Z}_{r} $

Now, take a sequence of i.i.d. blocks $ \{\tilde{\vc{Z}}_{O_j}:j=1,\cdots, \mu_t\} $ such that each $ \tilde{\vc{Z}}_{O_j} $ is independent of $ \{Z_t\}_{t=1}^T $ and each $ \tilde{\vc{Z}}_{O_j} $ has the same distribution as the corresponding block from the original sequence $ \{\vc{Z}_j: j\in O_j\} $. We construct the even and remainder blocks in a similar way and denote them  $ \{\tilde{\vc{Z}}_{E_j}:j=1,\cdots, \mu_t\} $ and  $ \tilde{\vc{Z}}_{Re} $ respectivey. 

$ \tilde{\vc{Z}}_O := \cup_{j=1}^{\mu_T}\tilde{\vc{Z}}_{O_j} $($ \tilde{\vc{Z}}_E := \cup_{j=1}^{\mu_T}\tilde{\vc{Z}}_{E_j} $) denote the union of the odd(even) blocks. 


For the odd blocks:
$\forall \, t>0$,

\begin{align*}
\mathbb{P}[ &\frac{2}{T} \vert  \Vert \vc{Z}_o\Vert^2_2 - \mathbb{E}(\Vert \vc{Z}_o\Vert^2_2)  \vert >t ]\\
&=\mathbb{E}[\mathbbm{1}\{ \frac{2}{T} \vert  \Vert \vc{Z}_o\Vert^2_2 - \mathbb{E}(\Vert \vc{Z}_o\Vert^2_2) \vert \} >t\}]\\
&\le \mathbb{E}[\mathbbm{1}\{ \frac{2}{T} \vert  \Vert \tilde{\vc{Z}_o}\Vert^2_2 - \mathbb{E}(\Vert \tilde{\vc{Z}_o}\Vert^2_2)\vert \}    >t\}] + (\mu_{a_T}-1) \beta(a_T)\\
& =\mathbb{P}[ \frac{2}{T} \vert  \Vert \tilde{\vc{Z}_o}\Vert^2_2 - \mathbb{E}(\Vert \tilde{\vc{Z}_o}\Vert^2_2)\vert    >t\}] + (\mu_{a_T}-1) \beta(a_T)\\
&= \mathbb{P}[\frac{1}{\mu_T} \vert \sum_{i=1}^{\mu_T} \Vert \tilde{\vc{Z}{o_i}}\Vert^2_2 - \mathbb{E}(\Vert \tilde{\vc{Z}_{o_i}}\Vert^2_2)    \vert >t a_T ]+ (\mu_{a_T}-1) \beta(a_T)\\
&\le2 \exp \bcur{  -C_B\min\bcur{ \frac{t^2 \mu_T}{K^2}, \frac{t \mu_T}{K}  }   }
+ (\mu_{a_T}-1) \beta(a_T) \label{eqn: subConc}
\end{align*}

Where the first inequality follows from  \cite[Lemma 4.1]{yu1994rates} with $M=1$. By  Fact \eqref{fact:subgaussExp}, the corresponding subexponential constant of each  $\vertii{\tilde{\vc{Z}_{o_i}}}^2$ $\le  a_T K$  where $K$ is the subexponential norm because of fact \ref{fact:subgaussExp}. With this, the second inequality follows from the Bernstein's inequality (Proposition \eqref{thm:Bernstein}) with some constant $C_B>0$.

Then 
\begin{align*}
2 \exp &\bcur{  -C_B\min\bcur{ \frac{t^2 \mu_T}{K^2}, \frac{t \mu_T}{K}  }   }
+ (\mu_{a_T}-1) \beta(a_T) \\
&\le
2 \exp \bcur{  -C_B\min\bcur{ \frac{t^2 \mu_T}{K^2}, \frac{t \mu_T}{K}  }   }
  + (\mu_{T}-1) \exp\{ -c_{\beta} a_T\}\\
\end{align*}

So,
$$
\mathbb{P}[ \frac{2}{T} \vert  \Vert \vc{Z}_o\Vert^2_2 - \mathbb{E}(\Vert \vc{Z}_o\Vert^2_2)  \vert >t ]
\le 
2 \exp \bcur{  -C_B\min\bcur{ \frac{t^2 \mu_T}{K^2}, \frac{t \mu_T}{K}  }   }
  + (\mu_{T}-1) \exp\{ -c_{\beta} a_T\}
$$
Taking the union bound over the odd and even blocks,
$$
\mathbb{P}[ \frac{1}{T} \vert  \Vert \vc{Z}\Vert^2_2 - \mathbb{E}(\Vert \vc{Z}\Vert^2_2)  \vert >t ]
\le 
4 \exp \bcur{  -C_B\min\bcur{ \frac{t^2 \mu_T}{K^2}, \frac{t \mu_T}{K}  }   }
  + 2(\mu_{T}-1) \exp\{ -c_{\beta} a_T\}
$$
For $0<t<K$, it reduces to
$$
\mathbb{P}[ \frac{1}{T} \vert  \Vert \vc{Z}\Vert^2_2 - \mathbb{E}(\Vert \vc{Z}\Vert^2_2)  \vert >t ]
\le 
4 \exp \bcur{  -C_B\frac{t^2 \mu_T}{K^2} }
  + 2(\mu_{T}-1) \exp\{ -c_{\beta} a_T\}
$$
For the remainder block, since $\Vert \vc{Z}_r\Vert^2_2$ has subexponential constant at most $a_T K \le KT/(2\mu_T)$, we have
$$
\mathbb{P} \bbra{ \frac{1}{T} \vert  \Vert \vc{Z}_r\Vert^2_2 - \mathbb{E}(\Vert \vc{Z}_r\Vert^2_2)  \vert 
>
t
}  
\le 
\exp\left( \frac{ - t T}{a_T K} \right)
\le
\exp\left( \frac{- 2 t \mu_T}{K}  \right)\\
$$
Together, by union bound
$$
\mathbb{P} \bbra{ \frac{1}{T} \vert  \Vert \vc{Z}\Vert^2_2 - \mathbb{E}(\Vert \vc{Z}\Vert^2_2)  \vert 
>
t
}  
\le 
4 \exp\{ -C_B \frac{t^2 \mu_T}{K^2}\} 
+
2(\mu_T-1) \exp\{-c_{\beta} a_T\}
+
\exp\{\frac{- 2 t \mu_T}{K}  \}
$$

\end{proof}

\begin{proof}[Proof of Proposition $\ref{results:REsub}$]

Recall that the sequence ${X}_1, \cdots, {X}_T \in \mathbb{R}^p$ form a $\beta$-mixing and stationary sequence. 

Now, fix a unit vector $\vc{v}\in \mathbb{R}^p, \; \Vert \vc{v}\Vert^2=1$. \\

Define real valued random variables $Z_t = X_t' \vc{v},\; t=1,\cdots,T$. Note that the $\beta$ mixing rate of $\{Z_t\}_{t=1}^T$ is bounded by the same of $\{X_t\}_{t=1}^T$ by Fact \ref{fact:mixingEquiv}. We suppress the $X$ subscript of the subgaussian constant $\sqrt{K_X}$ here, and refer it as $ \sqrt{K} $.\\

We can apply Lemma \ref{result: subgau:tailbdd} on $ {Z}:=\{Z_t\}_{t=1}^T $. Set $t = bK$.  We have,
\begin{align*}
\mathbb{P} \bbra{ \frac{1}{T} \vert  \Vert {Z}\Vert^2_2 - \mathbb{E}(\Vert {Z}\Vert^2_2)  \vert 
>
 bK} 
&\le
4\exp \bcur{
-C_B b^2\mu_T
}
+2(\mu_t-1)\exp \{
-c_{\beta}a_t
\}
+
 \exp \lbrace  - b\mu_T    \rbrace  \\
 &\le 
 5\exp \lbrace  -\min\{C_B, 2\} b^2\mu_T    \rbrace
 +
 2(\mu_t-1)\exp \{
 -c_{\beta}a_t
 \}
\end{align*}

Using Lemma F.2 in \cite{basu2015regularized}, we extend the inequality to hold for all vectors $\mathbb{J}(2k)$, the set of unit norm $2s$-sparse vectors. We have
$$
\mathbb{P} \bbra{\sup_{v \in \mathbb{J}(2k)} \frac{1}{T} \vert  \Vert {Z}\Vert^2_2 - \mathbb{E}(\Vert {Z}\Vert^2_2)  \vert 
> bK
}     
\le 
 5\exp \lbrace  -\csub b^2\mu_T   +3k\log(p) \rbrace
 +
 2(\mu_t-1)\exp \{
 -c_{\beta}a_t 
 +3k\log(p) 
 \}
$$
The constant $ \csub $ is defined as $\csub:=\min\{C_B,2\}$. 

Recall $\hat{\Gamma}:= \frac{\mt{X}'\mt{X}}{T}$, the above concentration can be equivalently expressed as
$$
\mathbb{P} \bbra{
\sup_{\vc{v} \in \mathbb{J}(2k)}
\verti{
\vc{v}'\bpar{
\hat{\Gamma}- \Sigma_X(0)
}\vc{v}
} 
\le
bK
}     
\ge
 1- 
  5\exp \lbrace  -\csub b^2\mu_T   +3k\log(p) \rbrace
  -
  2(\mu_t-1)\exp \{
  -c_{\beta}a_t 
  +3k\log(p) 
  \}
$$

Finally, we will extend the concentration to all $\vc{v}\in \R^p$ to establish the lower-RE result. By Lemma 12 of~\cite{loh2012high}, for parameter $k\ge 1$, w.p. at least

$$1- 
  5\exp \lbrace  -\csub b^2\mu_T   +3k\log(p) \rbrace
  -
  2(\mu_t-1)\exp \{
  -c_{\beta}a_t 
  +3k\log(p) 
  \}
  $$
  we have
$$
\verti{
\vc{v}'\bpar{
\hat{\Gamma}- \Sigma_X(0)
}\vc{v}
}
\le
27 Kb
\bbra{\vertii{\vc{v}}^2 + \frac{1}{k}\vertii{\vc{v}}^2_1}
$$

This implies that
$$
\vc{v}'  \hat{\Gamma}   \vc{v}
 \ge
\vertii{\vc{v}}^2 
\bbra{\lmin{\Sigma_X(0) }- 27  bK} 
- \frac{27 bK}{k}\vertii{\vc{v}}_1^2
$$
w.p. $
1- 
  5\exp \lbrace  -\csub b^2 \mu_T   +3k\log(p) \rbrace
  -
  2(\mu_t-1)\exp \{
  -c_{\beta}a_t 
  +3k\log(p) 
  \}
 $.\\

Now, choose set \black{$k =\frac{1}{6\log(p)} \min\bcur{\csub b^2\mu_t, c_{\beta} a_T}  $}. 
Let's choose that, for some $\xi \in (0,1)$, $a_t=T^\xi$ and $\mu_T = T^{1-\xi}$. Then,
$$
k=c\frac{1}{\log(p)}\min\{a_T,\mu_T \}=c\frac{1}{\log(p)}\min\{T^\xi,T^{1-\xi} \}
$$
Where \black{ $c= \frac{1}{6}\max\{c_{\beta},\csub b^2\}.$}
To ensure $k \ge 1$, we require $T \ge \bpar{\frac{1}{c}\log(p)}^{\min\bcur{\frac{1}{\xi}, \frac{1}{1-\xi}}} $

With these specifications, We have for probability at least
$$
1- 
  5\exp \lbrace  -\csub b^2 T^{\frac{1}{2}} \rbrace
  -
  2(T^{\frac{1}{2}}-1)\exp \{
  -c_{\beta}T^{\frac{1}{2}}/2 
    \}
 $$
that
$$
\vc{v}'{  \hat{\Gamma}   }\vc{v}
 \ge
\vertii{\vc{v}}^2 
\bbra{\lmin{\Sigma_X(0)} - 27bK} 
- 
\frac{27bK\log(p)}{c  \min\{T^\xi,T^{1-\xi} \}}
\vertii{\vc{v}}_1^2. \\
$$

Now, choose \black{$\xi =\half$} since it optimizes the rate of decay in the tolerance parameter. Also, choose \black{$b = \min\{\frac{1}{54K}\lmin{\Sigma_X(0)},1   \}$}; this ensures that $\lmin{\Sigma_X(0)} - 27bK \ge \half \lmin{\Sigma_X(0)}$. \\

In all, for $T \ge \bpar{\frac{1}{c}\log(p)}^2  $
w.p. at least
$$
1- 
  5\exp \lbrace  -\csub b^2 T^{\frac{1}{2}} \rbrace
  -
  2(T^{\frac{1}{2}}-1)\exp \{
  -c_{\beta}T^{\frac{1}{2}}/2 
    \}
 $$
 
$$
\vc{v}' \hat{\Gamma}   \vc{v}
 \ge
\vertii{\vc{v}}^2 
\half \lmin{\Sigma_X(0)} 
- 
\frac{27bK\log(p)}{c  T^\half}
\vertii{\vc{v}}_1^2. 
$$

\end{proof}


\begin{proof}[Proof of Proposition~\ref{result:betaDev}]
$ $\newline
Recall $\vertiii{\mt{X}'\mt{W} }_{\infty}= \max_{1\le i \le p,1\le j \le q } | [\mt{X}'\mt{W}]_{i,j} |= 
\max_{1\le i \le p,1\le j \le q }\verti{ \mt{X}_{:i}'\mt{W}_{:j}}$. \\

By lemma condition \eqref{as:0mean}, we have
\begin{align*}
&\mathbb{E} {\mt{X}_{:i}}=\vc{0},\forall i \;\;\;\text{and} \\
& 
 \mathbb{E} {\mt{Y}_{:j}}=\vc{0},\forall j  
\end{align*}

By first order optimality of the optimization problem in \eqref{eqn:bstar}, we have
$$
\mathbb{E} \mt{X}'_{:i}(\mt{Y}-\mt{X}\bstar)=\vc{0},\forall i \Rightarrow 
\mathbb{E} {\mt{X}_{:i}}'\mt{W}_{:j}=0,\forall i,j 
$$
We know $\forall i,j $
\begin{align*}
\verti{\mt{X}_{:i}'\mt{W}_{:j}} &= \verti{\mt{X}_{:i}'\mt{W}_{:j} - \E[ \mt{X}_{:i}'\mt{W}_{:j} ]}  \\
&=\frac{1}{2}\verti{ \left(\Vert \mt{X}_{:i}+ \mt{W}_{:j} \Vert^2 - \E[\Vert \mt{X}_{:i}+ \mt{W}_{:j} \Vert^2 ] \right) - \left( \Vert \mt{X}_{:i} \Vert^2 - \E[\Vert \mt{X}_{:i} \Vert^2 ] \right) - \left( \Vert \mt{W}_{:j} \Vert^2 -\E[ \Vert \mt{W}_{:j} \Vert^2] \right)} \\
&\le \half \verti{  \Vert \mt{X}_{:i}+ \mt{W}_{:j} \Vert^2 - \E[\Vert \mt{X}_{:i}+ \mt{W}_{:j} \Vert^2 ]  } + \half \verti{  \Vert \mt{X}_{:i} \Vert^2 - \E[\Vert \mt{X}_{:i} \Vert^2 ] } + \half \verti{  \Vert \mt{W}_{:j} \Vert^2 -\E[ \Vert \mt{W}_{:j} \Vert^2]  } 
\end{align*}
Therefore,
\begin{align*}
&\quad \prob \bpar{ \frac{1}{T} \verti{\mt{X}_{:i}'\mt{W}_{:j}} > 3t } \\
&\le \prob \bpar{ \frac{1}{2T} \verti{  \Vert \mt{X}_{:i}+ \mt{W}_{:j} \Vert^2 - \E[\Vert \mt{X}_{:i}+ \mt{W}_{:j} \Vert^2 ]  } > t }
+ \prob \bpar{ \frac{1}{2T} \verti{  \Vert \mt{X}_{:i} \Vert^2 - \E[\Vert \mt{X}_{:i} \Vert^2 ] } > t } \\
&\quad + \prob \bpar{ \frac{1}{2T} \verti{  \Vert \mt{W}_{:j} \Vert^2 -\E[ \Vert \mt{W}_{:j} \Vert^2] } > t } 
\end{align*}
This suggests proof strategy via controlling tail probability on each of the terms $  \verti{  \Vert \mt{X}_{:i} \Vert^2 - \E[\Vert \mt{X}_{:i} \Vert^2 ] } $,  $ \verti{  \Vert \mt{W}_{:j} \Vert^2 -\E[ \Vert \mt{W}_{:j} \Vert^2] }$ and $  \verti{  \Vert \mt{X}_{:i}+ \mt{W}_{:j} \Vert^2 - \E[\Vert \mt{X}_{:i}+ \mt{W}_{:j} \Vert^2 ]  }$. Assuming the conditions in lemma \ref{result:betaDev}, we  can apply lemma \ref{result: subgau:tailbdd} on each of them. We have to figure out their subgaussian constants.\\

Let's define  $  K_\mt{W} := \sup_{1\le t \le T, 1\le j \le q} \snorm{\mt{W}_{tj}} 	 $ and $  K_{\mt{X}+\mt{W}} := \sup_{1\le t \le T, 1\le j \le q, 1\le i \le p} \snorm{\mt{X}_{ti}+ \mt{W}_{tj}} 	 $. We have to figure out the constants $ K_{\mt{W}} $ and $  K_{\mt{W}+\mt{X}}  $.

Now,
\begin{align*}
 \sup_{1\le t\le T}\sup_{1\le i\le q} \snorm{\mt{W}_{ti}} 
&\le \sup_{1\le t\le T} \snorm{\mt{W}_{t:}}  & &\text{by definition of subgaussian random vector} \\
&=\snorm{\mt{W}_{1:}}						& &\text{by stationarity}
\end{align*}

 Let's figure out $ \snorm{\mt{W}_{1:}}	 $,
\begin{align*}
{\mt{W}_{1:}} &=\mt{Y}_{1:} -(\mt{X}\bstar)_{1:} \\
				&=\mt{Y}_{1:} -{\mt{X}_{1:}}\bstar \\
\end{align*}
Thus,
\begin{align*}
\snorm{\mt{W}_{1:}} &\le \snorm{\mt{Y}_{1:}} +\snorm{{\mt{X}_{1:}}\bstar} & &\text{since $ \snorm{\cdot} $ is a norm} \nonumber\\
				&\le \snorm{\mt{Y}_{1:}} +\snorm{{\mt{X}_{1:}}}\vertiii{\bstar} &&\text{by lemma \ref{result:subadd:subgauss}}\nonumber\\
				&=\sqrt{K_{Y}} +\vertiii{\bstar}  \sqrt{K_{X}} & &\text{by stationarity}
\end{align*}

Therefore,
\begin{equation}\label{subGconst:W}
K_\mt{W} \le \sqrt{K_{Y}} +\vertiii{\bstar} \sqrt{K_{X}}
\end{equation}

Similarly,
\begin{align*}
\sup_{1\le i \le p, 1\le j \le q, 1\le t \le T}\snorm{\mt{X}_{ti} + \mt{W}_{tj}} &\le \sup_{1\le i \le p, 1\le t \le T}\snorm{\mt{X}_{ti} }+  \sup_{1\le j \le q,1\le t \le T}\snorm{\mt{W}_{tj}} \\			
						&\le \snorm{\mt{X}_{1:}} + \snorm{\mt{W}_{1:}} \nonumber\\
						&\le \sqrt{K_{Y}} + \sqrt{K_{X}}\bpar{1+\vertiii{\bstar} } & &\text{by equation \eqref{subGconst:W}			}
\end{align*}
Therefore,
\begin{equation}\label{subGconst:XplusW}
K_{\mt{X}+\mt{W}} \le \sqrt{K_{Y}} + \sqrt{K_{X}}\bpar{1+\vertiii{\bstar} } \\
\end{equation}

Take 
\begin{eqnarray}\label{subGconst:max}
 K:=\max\{K_\mt{X}, K_\mt{W}, K_{\mt{X}+\mt{W}}\}\le \sqrt{K_{Y}} + \sqrt{K_{X}}\bpar{1+\vertiii{\bstar} }  
\end{eqnarray}

 For $\xi \in[0,1]$, set $a_T=T^\xi$ and $\mu_T=T^{1-\xi}$. Applying lemma \ref{result: subgau:tailbdd} three times with subgaussian constant $ K $, we have
\begin{align*}
\prob \bpar{ \frac{1}{T} \verti{\mt{X}_{:i}'\mt{W}_{:j}} > 3t }
&\le \prob \bpar{ \frac{1}{2T} \verti{  \Vert \mt{X}_{:i}+ \mt{W}_{:j} \Vert^2 - \E[\Vert \mt{X}_{:i}+ \mt{W}_{:j} \Vert^2 ]  } > t }
+ \prob \bpar{ \frac{1}{2T} \verti{  \Vert \mt{X}_{:i} \Vert^2 - \E[\Vert \mt{X}_{:i} \Vert^2 ] } > t } \\
&\quad + \prob \bpar{ \frac{1}{2T} \verti{  \Vert \mt{W}_{:j} \Vert^2 -\E[ \Vert \mt{W}_{:j} \Vert^2] } > t } \\
 &\le \ 4 \exp\{ -C_B \frac{4t^2 T^{1-\xi}}{K^4}\} 
+
2(T^{1-\xi}-1) \exp\{-c_{\beta} T^\xi\}
+
\exp\{-  \frac{2}{K^2} t T^{1-\xi} \}\} \\
&\quad+ 4 \exp\{ -C_B \frac{4t^2 T^{1-\xi}}{K^4}\} 
+
2(T^{1-\xi}-1) \exp\{-c_{\beta} T^\xi\}
+
\exp\{-  \frac{2}{K^2} t T^{1-\xi} \} \} \\
&\quad+ 4 \exp\{ -C_B \frac{4t^2 T^{1-\xi}}{K^4}\} 
+
2(T^{1-\xi}-1) \exp\{-c_{\beta} T^\xi\}
+
\exp\{-  \frac{2}{K^2} t T^{1-\xi} \} \}
\end{align*}

By union bound, 
\begin{align*}
\mathbb{P}[ \frac{1}{T} &\vertiii{\mt{X}'\mt{W}}_{\infty} > 3t ] 
= 
\mathbb{P}[ \max_{1\le i \le p, \, 1\le j \le q}\frac{1}{T} \vert  \mt{X}_{:i}'\mt{W}_{:j}  \vert > 3t ] \\
&\le
3pq\bcur{
4 \exp\{ -C_B \frac{4t^2 T^{1-\xi}}{K^4}\} 
+
2(T^{1-\xi}-1) \exp\{-c_{\beta} T^\xi\}
+
\exp\{-  \frac{2}{K^2} t T^{1-\xi} \}
} 
\\
&=
3\bcur{
4 \exp\{ -C_B \frac{4t^2 T^{1-\xi}}{K^4} + \log\{pq \} \} 
+
2(T^{1-\xi}-1) \exp\{-c_{\beta} T^\xi +\log\{pq \}\}
+ \exp\{-  \frac{2}{K^2} t T^{1-\xi} +\log\{pq \}\}
} 
\end{align*}

To ensure proper decay in the probability, we require
$$
\black{
T\ge \max \bcur{
\bpar{ \log(pq)  \max\bcur{ \frac{K^4}{2C_B},K^2 } }^{\frac{1}{1-\xi}},
\bbra{\frac{2}{c_{\beta}}\log(pq)   }^\frac{1}{\xi} ,
}
}
$$

With 
\black{$$t:= \sqrt{\frac{K^4 \log(pq)}{2T^{1-\xi}C_B}}$$}
\begin{align*}
\mathbb{P}
\bbra{ \frac{1}{T} \vertiii{\mt{X}'\mt{W}}_{\infty} 
> 
\sqrt{\frac{72K^4 \log(pq)}{T^{1-\xi}C_B}}
}
&\le
15\exp\bcur{ - \half \log(pq)} 
+
6(T^{1-\xi}-1) \exp\bcur{-\half c_{\beta}T^{\xi}  }
\\
\end{align*}
where 
\black{
$  K=\sqrt{K_\mt{Y}} + \sqrt{K_X}\bpar{1+\vertiii{\bstar} }  $
}  

\end{proof}


\begin{lm}\label{result:subadd:subgauss}
For any subgaussian random vector $\vc{X}$  and non-stochastic matrix $\mt{A}$. We have

$$
\vertii{\mt{A} \vc{X}}_{\psi_2}
\le \vertiii{\mt{A}}\vertii{\vc{X}}_{\psi_2}
$$
\end{lm}
\begin{proof}
%
%
We have,
\begin{align*}
\vertii{\mt{A} \vc{X}}_{\psi_2} &= 
\sup_{\vertii{\vc{v}}_2\le 1}\vertii{\vc{v}'\mt{A} \vc{X}}_{\psi_2}\\
&= \sup_{\vertii{\vc{v}}_2\le 1}\vertii{(\mt{A}'\vc{v})' \vc{X}}_{\psi_2}\\
&\le \sup_{\vertii{\vc{u}}_2\le \vertiii{\mt{A}}}\vertii{\vc{u}' \vc{X}}_{\psi_2}\\
&= \vertiii{\mt{A}} \sup_{\vertii{u}_2\le 1}\vertii{\vc{u}' \vc{X}}_{\psi_2}\\
&=  \vertiii{\mt{A}}\vertii{\vc{X}}_{\psi_2} .
\end{align*}
\end{proof}


\section{Bernstein's Concentration Inequality}
We state the Bernstein's inequality \citep[Proposition 5.16]{vershynin2010introduction} below for completeness.
\begin{pr}[Bernstein's Inequality]\label{thm:Bernstein}
Let $ X_1, \cdots, X_N $ be independent centered subexponential random variables, and $ K = \max_i \enorm{X_i} $
. Then for every $ a =(a_1, \cdots , a_N ) \in \R^N $ and every $ t \ge 0 $, we have
\[
\prob \bcur{
\verti{
\sum_{i=1}^{N}a_iX_i
}
\ge 
t
}
\le
2\exp\bbra{
-C_B\min\bpar{
\frac{t^2}{K^2\vertii{a}^2_2},
\frac{t}{K\vertii{a}_{\infty}}
}
}
\]
where $C_B>0 $ is an absolute constant.
\end{pr}

\section{Verification of Assumptions for the Examples}\label{Apnx:Ver}
\subsection{VAR}\label{veri:VAR}

Formally a finite order Gaussian VAR($ d $) process is defined as follows.
Consider a sequence of serially ordered random vectors $(Z_t)_{t=1}^{T+d}$, ${Z_t}\in \R^p$ that admits the following auto-regressive representation:
\begin{align}\label{eq:VAR(1)}
{Z_t} = \mt{A}_1 {Z}_{t-1}+ \dots +  \mt{A}_d {Z}_{t-d} + {\mathcal{E}}_t
\end{align}
where each $\mt{A}_k, k=1, \dots, d$ is  a non-stochastic coefficient matrix in $\R^{p \times p}$ and innovations ${\mathcal{E} }_t$ are $p$-dimensional random vectors from $\mathcal{N}( \vc{0}, \Sigma_{\epsilon})$. Assume $\lmin{\Sigma_{\epsilon}}>0$ and $\lmax{\Sigma_{\epsilon}}< \infty$.

Note that every VAR(d) process has an equivalent VAR(1) representation (see e.g. \cite[Ch 2.1]{lutkepohl2005new}) as 
 \begin{align}\label{eq:VAR(d)}
 \tilde{Z}_t = \tilde{\mt{A}}\tilde{Z}_{t-1} + \tilde{\mathcal{E}}_t
 \end{align}
  where 
 \begin{align}\label{eq:VAR1}
 \tilde{Z_t}:=
 \begin{bmatrix}
 \vc{Z}_t \\ 
 \vc{Z}_{t-1} \\
 \vdots \\
\vc{ Z}_{t-d+1}
  \end{bmatrix}_{(pd \times 1)} 
  &
  \tilde{\mathcal{E}}_t:=
\begin{bmatrix}
\mathcal{\vc{}E}_t \\ 
\vc{0} \\
 \vdots \\
\vc{0} 
\end{bmatrix}_{(pd \times 1)}  
\text{and }\;\;
\tilde{A}:=
\begin{bmatrix}
\mt{A}_1 & \mt{A}_2 & \cdots & \mt{A}_{d-1} & \mt{A}_d \\ 
\mt{I}_p &  \vc{0}   &     \vc{0}    &      \vc{0}  & \vc{0}  \\ 
\vc{0}     & \mt{I}_p &        &   \vc{0}      & \vc{0}  \\ 
\vdots    &     & \ddots    &  	\vdots     & \vdots\\ 
\vc{0}     &  \vc{0}    &    \cdots    &     \mt{I}_p & \vc{0} 
\end{bmatrix}_{(dp \times dp)} 
 \end{align}
Because of this equivalence, justification of Assumption~\ref{assum:beta} will operate through this corresponding augmented VAR$(1)  $ representation.

For both Gaussian and sub-Gaussian VARs, Assumption~\ref{as:0mean} is true since the sequences $ (Z_t) $ is centered. Second,  $ \bstar=(\mt{A}_1, \cdots, \mt{A}_d) $. So Assumption~\ref{as:spars}  follows from construction.

For the remaining Assumptions, we will consider the Gaussian and sub-Gaussian cases separately.

\paragraph{Gaussian VAR}
$ (Z_t) $ satisfies Assumption~\ref{assum:subgauss} by model assumption.

To show that $({Z}_t)   $ is $ \beta $-mixing with geometrically decaying coefficients, we use the following facts together with the equivalence between $ (Z_t) $ and $ (\tilde{Z}_t) $ and Fact \ref{fact:mixingEquiv}.

Since $ (\tilde{Z}_t) $ is stable, the spectral radius of $ \tilde{A} $, $r(\tilde{A})<1  $, hence Assumption~\ref{as:stat} holds. Also the innovations $ \tilde{\mathcal{E}} $ has finite  first absolute moment and positive support everywhere. Then, according to  Theorem 4.4 in \cite{tjostheim1990non},  $ (\tilde{Z}_t) $ is \textit{geometrically ergodic}. Note here that Gaussianity is \emph{not} required here. Hence, it also applies to innovations from mixture of Gaussians.
 
Next, we present a standard result (see e.g. \cite[Proposition 2]{liebscher2005towards}). 
 
\begin{fact}\label{fact:ergodic}
A stationary Markov chain $\{\sv{Z}_t\}$ is geometrically ergodic implies 
 $\{\sv{Z}_t\}$ is \textit{absolutely regular}(a.k.a. $ \beta $-mixing) with 
 $$
 \beta(n)=O(\gamma^n),\,\, \gamma^n \in(0,1)
 $$
\end{fact}
 
So, Assumption~\ref{assum:beta} holds. 

\paragraph{Sub-Gaussian VAR}
When the innovations are random vectors from the uniform distribution, they are sub-Gaussian. That $ (Z_t) $ are sub-Gaussian follows from arguments as in Appendix \ref{veri:ARCH} with $ \Sigma(\cdot) $ set to be the idenity operator in this case. So, Assumption~\ref{assum:subgauss} holds. 

To show that $({Z}_t)   $ satisfies Assumptions~\ref{as:stat} and~\ref{assum:beta}, we establish that $({Z}_t)   $ is geometrically ergodic. To show the latter, we use Propositions 1 and 2 in~\cite{liebscher2005towards} together with the equivalence between $ (Z_t) $ and $ (\tilde{Z}_t) $ and Fact \ref{fact:mixingEquiv}.

To apply Proposition 1 in \cite{liebscher2005towards}, we check the three conditions one by one. Condition (i) is immediate with $ m=1,\, E=\R^p$, and $\mu$ is the Lebesgue measure. For condition (ii), we set $ E=\R^p$, $\mu$ to be the Lebesgue measure, and $\bar{m}= \ceil{\inf_{\vc{u} \in C, \vc{v} \in A}\vertii{u-v}_2} $ the minimum ``distance" between the sets $ C   $ and $ A $. Because $ C  $ is bounded and $ A $ Borel, $ \bar{m} $ is finite. Lastly, for condition (iii), we again let  $ E=\R^p$, $\mu$ to be the Lebesgue measure, and now the function $ Q(\cdot)= \vertii{\cdot} $ and the set $ K=\{x \in \R^p: \vertii{x}\le \frac{2\E\vertii{\tilde{\mathcal{E}}_t}}{c\,\epsilon} \} $ where $ c=1-\vertiii{\tilde{A}} $. Then, 

\begin{itemize}
\item 
Recall from model assumption that $ \vertiii{\tilde{A}}<1 $;  hence, 
\begin{align*}
&\E \bbra{\vertii{\,\tilde{Z}_{t+1} }\,\middle| \tilde{Z}_t=z} <  \vertiii{\tilde{A}} \vertii{z} + \E(\vertii{\tilde{\mathcal{E}}_{t+1}}) 
\le \bpar{1-\frac{c}{2}}\vertii{z} -\epsilon ,\; \\
&\text{for all } z \in E\backslash K
\end{align*} 
\item
For all $ z \in K $,
\begin{align*}
\E \bbra{\vertii{\,\tilde{Z}_{t+1} }\,\middle| \tilde{Z}_t=z} <  \vertiii{\tilde{A}} \vertii{z} + \E(\vertii{\tilde{\mathcal{E}}_{t+1}}) \le \vertiii{\tilde{A}}\frac{2\E\vertii{\tilde{\mathcal{E}}_t}}{c\epsilon}
\end{align*}
\item
For all $ z \in K $,
\begin{align*}
0 \le \vertii{z } \le \frac{2\E\vertii{\tilde{\mathcal{E}}_t}}{c\epsilon} 
\end{align*}
\end{itemize}

Now, by Proposition 1 in~\cite{liebscher2005towards}, $ (\tilde{Z}_t) $ is geometrically ergodic; hence $ (\tilde{Z}_t) $ will be stationary. Once it reaches stationarity, by Proposition 2 in the same paper, the sequence will be $ \beta $-mixing with geometrically decaying mixing coefficients. 
Therefore, Assumptions~\ref{as:stat} and~\ref{assum:beta} hold.

\subsection{VAR with Misspecification}\label{veri:misVAR}

%
%

\textbf{Assumptions}:
Assumption~\ref{as:0mean} is immediate from model definitions. By the same arguments as in Appendix \ref{veri:VAR}, $ (Z_t,\Xi_t) $ are stationary and so is the sub-process $ (Z_t) $; Assumption~\ref{as:stat} holds. Again,   $ (Z_t,\Xi_t) $ 
satisfy  Assumption~\ref{assum:beta} according to Appendix \ref{veri:VAR}. By Fact \ref{fact:mixingEquiv}, we have the same Assumptions hold for the respective sub-processes $ (Z_t) $ in both cases. Assumption~\ref{assum:subgauss} holds by the same reasoning as in Appendix~\ref{veri:VAR}.

To show that $ (\bstar)'=A_{Z Z }+A_{Z \Xi } \Sigma_{\Xi Z }(0)(\Sigma_Z (0))^{-1} $, consider the following arguments. 
By Assumption~\ref{as:stat}, we have
the auto-covariance matrix of the whole system $ (Z_t,\Xi_t) $ as

\[
\Sigma_{(Z,\,\Xi)}
=
\begin{bmatrix}
\Sigma_X(0) &  \Sigma_{X\Xi}(0)\\ 
\Sigma_{\Xi X}(0) & \Sigma_\Xi (0)
\end{bmatrix} 
\]
Recall our $ \bstar $ definition from Eq.~\eqref{eqn:bstar}
$$
\bstar := \argmin_{B \in \R ^{p\times p}} \E \bpar{  
\vertii{
Z _t - B' Z _{t-1}
}^2_2
}\\
$$
Taking derivatives and setting to zero, we obtain
\begin{equation} \label{eq:tilA}
(\bstar)' = \Sigma_Z (-1) (\Sigma_Z )^{-1}
\end{equation}
Note that 
\begin{align*}
\Sigma_Z (-1)	&= \Sigma_{(Z ,\,\Xi )}(-1)[1:p_1, 1:p_1] \\
				&= \E \bpar{A_{Z Z }Z _{t-1} + A_{Z \Xi }\Xi _{t-1} + \mathcal{E}_{Z ,t-1}} Z_{t-1}' \\
				&=\E\bpar{A_{Z Z }Z_{t-1} Z_{t-1}' +  A_{Z \Xi }\Xi _{t-1}  Z _{t-1}' +\mathcal{E}_{Z ,t-1}  Z_{t-1}' }\\
				&= A_{Z Z }\Sigma_Z (0) +A_{Z \Xi } \Sigma_{\Xi Z }(0)
\end{align*}
by Assumptions \ref{as:stat} and the fact that the innovations are iid.\\

Naturally, 
$$
(\bstar)'
= A_{Z Z }\Sigma_Z (0)(\Sigma_Z (0))^{-1} +A_{Z \Xi } \Sigma_{\Xi Z }(0)(\Sigma_Z (0))^{-1}
=A_{Z Z }+A_{Z \Xi } \Sigma_{\Xi Z }(0)(\Sigma_Z (0))^{-1}
$$

\begin{rem}
Notice that $A_{Z \Xi }   $ is a column vector and suppose it is $ 1 $-sparse, and $ A_{Z Z } $ is $ p $-sparse, then $ \bstar$ is at most $  2p$-sparse. So Assumption~\ref{as:spars} can be built in by model construction. 
\end{rem}

\begin{rem}
We gave an explicit model here where the left out variable $ \Xi  $ was univariate. That was only for convenience. In fact, whenever the set of left-out variables $ \Xi  $ affect only a small  set of variables $ \Xi  $ in the retained system $ Z  $, the matrix $\bstar$ is guaranteed to be sparse. To see that, suppose $ \Xi \in \R^q $ and $ A_{Z \Xi } $ has at most $s_0$ non-zero rows (and let $ A_{Z Z } $ to be $ s $-sparse as always), then $\bstar$ is at most $ (s_0 p +s )$-sparse.
\end{rem}
\begin{rem}
Any VAR($d$) process has an equivalent VAR(1) representation (Lutkepohl 2005). Our results extend to any VAR($d$) processes. 
\end{rem}


\subsection{ARCH}\label{veri:ARCH}

\paragraph{Verifying the Assumptions. }

To show that Assumption~\ref{assum:beta} holds for a process defined by Eq. \eqref{eq:ARCH} we leverage on Theorem 2 from~\cite{liebscher2005towards}. Note that the original ARCH model in~\cite{liebscher2005towards} assumes the innovations to have positive support everywhere. However, this is just a convenient assumption to establish the first two conditions in Proposition 1 (on which proof of Theorem 2 relies) from the same paper. Our example ARCH model with innovations from the uniform distribution also satisfies the first two conditions of Proposition 1 by the same arguments  in the \emph{Sub-Gaussian} paragraph of Appendix~\ref{veri:VAR}. 

Theorem 2 tells us that for our ARCH model, if it satisfies the following conditions, it is guaranteed to be 
absolutely regular with geometrically decaying $ \beta $-coefficients.
\begin{itemize}
\item $\mathcal{E}_t$ has positive density everywhere on $\R^p$ and has identity covariance by construction.
\item $\Sigma(\vc{z}) = o(\vertii{\vc{z}})$ because $m \in (0,1)$.
\item $\vertiii{ \Sigma(\vc{z})^{-1} } \le 1/(ac)$, $|\mathrm{det} \left( \Sigma(\vc{z}) \right) | \le bc$
\item $r(A) \le \vertiii{ A } < 1$
\end{itemize}
So, Assumption~\ref{assum:beta} is valid here. We check other assumptions next.

Mean $ 0 $ is immediate, so we have Assumption~\ref{as:0mean}.
When the Markov chain did not start from a stationary distribution, geometric ergodicity implies that the sequence is approaching the stationary distribution exponentially fast. So, after a burning period, we will have Assumption \ref{as:stat} approximately valid here. 

The sub-Gaussian constant of $\Sigma(Z_{t-1})\mathcal{E}_t$ given $ Z_{t-1}=\vc{z} $ is bounded as follows:
for every $ \vc{z} $,
\begin{align*}
\snorm{ \Sigma(\vc{z})\mathcal{E}_t } &\le \vertiii{ \Sigma(\vc{z}) } \snorm{\mathcal{E}_t} &&\text{by Lemma~\ref{result:subadd:subgauss}}\\
&\le C \vertiii{ \Sigma(\vc{z}) }  \cdot  \snorm{e_1' \mathcal{E}_t} \\
&\le C  \vertiii{ \Sigma(\vc{z}) }  \cdot  \snorm{\U{-\sqrt{3},\sqrt{3}}}  \\
&\le C'cb=: K_E
\end{align*}
The second inequality follows since $ \mathcal{E}_t \overset{iid}{\sim} \U{\bbra{-\sqrt{3},\sqrt{3}}^p} $ and a standard result that
\begin{fact}
Let $ X=(X_1,\cdots,X_p)\in \R^p $ be a random vector with independent, mean zero, sub-Gaussian coordinates $ X_i $. Then $ X $ is a sub-Gaussian random vector, and there exists a positive constant $ C $ for which
\begin{align*}
\snorm{X}\le C \cdot \max_{i\le p}\snorm{X_i}
\end{align*}
\end{fact}
The forth inequality follows since the sub-Gaussian norm of a bounded random variable is also bounded. 

By the recursion for $Z_t$, we have
\[
\snorm{Z_t} \le \vertiii{ A } \snorm{Z_{t-1}} +  K_E .
\]
which yields the bound $\snorm{Z_t} \le K_E / (1-\vertiii{ A } ) < \infty$. Hence Assumption~\ref{assum:subgauss} holds.

We will show below that $\bstar=A'$. Hence, sparsity (Assumption~\ref{as:spars}) can be built in when we construct our model \ref{eq:ARCH}. 

Recall Eq. \ref{eq:tilA} from Appendix~\ref{veri:misVAR} that 
$$
\bstar= \Sigma_Z(-1) (\Sigma_Z)^{-1}
$$
Now,
\begin{align*}
\Sigma_Z(-1)&=\E Z_t Z_{t-1}'&&\text{by stationarity} \\
			&= \E \bpar{  A Z_{t-1} +\Sigma(Z_{t-1}) \mathcal{E}_t }Z_{t-1}' &&\text{Eq.~\eqref{eq:ARCH}}\\
			&=A \E Z_{t-1}Z_{t-1}' + \E \Sigma(Z_{t-1}) \mathcal{E}_t Z_{t-1 }' \\
			&= A\Sigma_Z  + \E[ c\,\clip{\vertii{Z_{t-1}}^{m}}{a}{b}  \mathcal{E}_t Z_{t-1 }'] \\
			&= A\Sigma_Z  +  \E[ c\mathcal{E}_t Z_{t-1 }' \clip{\vertii{Z_{t-1}}^{m}}{a}{b}]\\
			&= A\Sigma_Z  +   c\E\bbra{ \mathcal{E}_t} \E\bbra{ Z_{t-1 }'  \clip{\vertii{Z_{t-1}}^{m}}{a}{b} } &&\text{i.i.d. innovations}\\
			&=A\Sigma_Z  &&\text{$ \mathcal{E}_t $ mean $ 0 $} ,
\end{align*}
where $\clip{x}{a}{b} := \min\{\max\{x,a\},b\}$ for $b > a$.

Since $ \Sigma_Z $ is invertible, we have $ (\bstar)'= \Sigma_Z(-1) (\Sigma_Z)^{-1} =A $.

\subsection*{Acknowledgments}
We thank Sumanta Basu and George Michailidis for helpful discussions, and Roman Vershynin for pointers to the literature. We acknowledge the support of NSF via a regular (DMS-1612549) and a CAREER grant (IIS-1452099).

\bibliographystyle{plain}
\bibliography{myBib,time_series_grant}

\end{document}